\documentclass[11pt]{article}
\usepackage[utf8]{inputenc}

\usepackage[dvipsnames]{xcolor}
\usepackage[T1]{fontenc}
\usepackage{array}
\usepackage[utf8]{inputenc}
\usepackage[english]{babel}
\usepackage{pgfplots}
\usepackage{graphicx,wrapfig,lipsum}
\usepackage{float}    
\usepackage{verbatim} 
\usepackage{amsmath, amssymb, amsthm, mathtools}  
\usepackage{caption}
\usepackage{subcaption}
\usepackage[colorlinks,linkcolor=purple, citecolor=Green, urlcolor=magenta, anchorcolor=ForestGreen,bookmarks=False]{hyperref}
\usepackage{cleveref}
\usepackage{bookmark}
\usepackage{fullpage}
\usepackage{enumerate}
\usepackage{paralist}
\usepackage{xspace}
\usepackage{bbm}
\usepackage{bm}
\usepackage{url}
\usepackage{lipsum}
\usepackage{algorithm}
\usepackage{algpseudocode} 
\usepackage{color}
\usepackage{microtype}
\usepackage[section]{placeins}
\usepackage{lscape}
\usepackage{multirow}
\usepackage{todonotes}
\makeatletter

\newcommand{\norm}[1]{\left\lVert#1\right\rVert}

\newcommand{\tr}{{\operatorname{tr}}}

\newcommand{\cE}{\mathcal{E}}

\newcommand{\cM}{\mathcal{M}}
\newcommand{\cN}{\mathcal{N}}
\newcommand{\cO}{\mathcal{O}}
\newcommand{\cP}{\mathcal{P}}

\newcommand{\cS}{\mathcal{S}}

\newcommand{\cV}{\mathcal{V}}

\newcommand{\cX}{\mathcal{X}}

\newtheorem{theorem}{Theorem}

\newtheorem{lemma}{Lemma}

\newtheorem{proposition}{Proposition}

\newtheorem{assumption}{Assumption}

\def \RR {\mathbb{R}}
\def \Pr {\mathbb{P}}
\def \SS {\mathbb{S}}

\newcommand{\EE}{\mathbb{E}}

\DeclareMathOperator*{\argmin}{argmin}
\DeclareMathOperator*{\col}{\mathbf{col}}

\def \Sstar {\mathcal{S}^\star}
\def \rstar {r^\star}
\def \Sigmastar {\Sigma^\star}

\def \gammastar {\gamma^\star}
\def \Ustar {U^\star}
\def \Dstar {D^\star}

\def \median {\mathrm{median}}

\title{RANSAC Revisited: An Improved Algorithm for Robust Subspace Recovery under Adversarial and Noisy Corruptions}

\author{
Guixian Chen\\
University of Michigan\\
\texttt{gxchen@umich.edu}\\
\and 
Jianhao Ma\\
University of Michigan\\
\texttt{jianhao@umich.edu}
\and
Salar Fattahi\\
University of Michigan\\
\texttt{fattahi@umich.edu}
}

\begin{document}

\maketitle

\begin{abstract}
    In this paper, we study the problem of robust subspace recovery (RSR) in the presence of both strong adversarial corruptions and Gaussian noise. Specifically, given a limited number of noisy samples---some of which are tampered by an adaptive and strong adversary---we aim to recover a low-dimensional subspace that approximately contains a significant fraction of the uncorrupted samples, up to an error that scales with the Gaussian noise. Existing approaches to this problem often suffer from high computational costs or rely on restrictive distributional assumptions, limiting their applicability in truly adversarial settings.
To address these challenges, we revisit the classical random sample consensus (RANSAC) algorithm, which offers strong robustness to adversarial outliers, but sacrifices efficiency and robustness against Gaussian noise and model misspecification in the process. We propose a two-stage algorithm, RANSAC+, that precisely pinpoints and remedies the failure modes of standard RANSAC. Our method is provably robust to both Gaussian and adversarial corruptions, achieves near-optimal sample complexity without requiring prior knowledge of the subspace dimension, and is more efficient than existing RANSAC-type methods. 
\end{abstract}

\section{Introduction}
\label{sec:intro}
Modern datasets are generated in increasingly high dimensions and vast quantities. A fundamental approach to analyzing such large-scale, high-dimensional datasets is to represent them using a potentially low-dimensional subspace—a process known as {\it subspace recovery}. By projecting data onto this subspace, we can significantly reduce its dimensionality while preserving its essential structure and information.
Motivated by this, we focus on the fundamental problem of {\it robust subspace recovery} (RSR): \vspace{2mm}

\noindent{\it Given a collection of $n$ points in $\mathbb{R}^d$---some of which may be corrupted by adversarial contamination, and all of which are perturbed by small noise---how can we efficiently and provably identify a subspace $\cS^\star$ of dimension $\dim(\cS^\star) = r^\star\ll d$ that nearly contains a significant portion of the uncorrupted samples?}\vspace{2mm}

In this context, we consider one of the most challenging corruption model, namely the \textit{adversarial and noisy contamination model}, formally defined as follows.

\begin{assumption}[Adversarial and noisy contamination model]
\label{dfn:adversarial_contamination}
Given a clean distribution $\cP$, a corruption parameter $\epsilon \in (0, \epsilon_0)$\footnote{Here $\epsilon_0$ is the breakdown point, which varies for different estimators. Roughly speaking, breakdown point is the largest fraction of contamination the data may contain when the estimator still returns a valid estimation.}, and a noise covariance $\Sigma_\xi$, the $(\epsilon,\Sigma_\xi)$-corrupted samples are generated as follows: (i) \( n \) i.i.d. samples are drawn from \( \mathcal{P} \). These samples are referred to as \textit{clean samples}. (ii) The adversary adds i.i.d. noise to each sample, drawn from a zero-mean Gaussian distribution with covariance \( \Sigma_\xi \). The resulting samples are referred to as \textit{inliers}. (iii) An arbitrarily powerful adversary inspects the samples, removes \( \lfloor \epsilon n \rfloor \) of them, and replaces them with arbitrary points, referred to as \textit{outliers}. 
\end{assumption}

The above contamination model generalizes a variety of existing models, including Huber's contamination model \cite[Section 5-7]{huber1964robust}, and the noiseless adversarial contamination model \cite[Definition 1.6]{diakonikolas2023algorithmic} (corresponding to $\Sigma_\xi = 0$).

Specifically, we consider scenarios where the clean samples drawn from the distribution \( \mathcal{P} \) reside in a low-dimensional space. Formally, this is captured by the following assumption.

\begin{assumption}[Distribution of clean samples]
\label{dfn:inliers}
The clean probability distribution $\cP$ is absolutely continuous, with a mean of zero\footnote{The zero-mean assumption on $\cP$ is made with minimal loss of generality. Specifically, if the samples $\cX = \{x_1,\dots, x_{2n}\}\subset \mathbb{R}^d$ are $(\epsilon, \Sigma_\xi)$-corrupted and generated from a distribution $\cP$ with an unknown, potentially nonzero mean, the samples $\widehat{\cX} = \{x_1-x_2,\dots, x_{2n-1}-x_{2n}\}\subset \mathbb{R}^d$ will be $(2\epsilon, 2\Sigma_\xi)$-corrupted from a zero-mean distribution $\cP'$. Furthermore, the clean samples will lie in the same $\rstar$-dimensional subspace as those of $\cX$. } and an unknown covariance matrix $\Sigma^\star \in \mathbb{R}^{d \times d}$, where the rank of $\Sigma^\star$, denoted as $\mathrm{rank}(\Sigma^\star) = \rstar\leq d$, is unknown.
\end{assumption}

Specifically, given an \( (\epsilon, \Sigma_\xi) \)-corrupted sample set as per \Cref{dfn:adversarial_contamination}, where the clean distribution satisfies \Cref{dfn:inliers}, our goal is to design an efficient method to recover the \( r^\star \)-dimensional subspace spanned by the clean samples, with a recovery error proportional to the magnitude of the Gaussian noise while remaining independent of the nature of the outliers.

\subsection{Failure of the existing techniques}\label{subsec::failure} 
One of the earliest algorithms for RSR is the \textit{random sample consensus} (RANSAC) algorithm, originally introduced in the 1980s~\cite{fischler1981random}. Despite its provable robustness to adversarial contamination, RANSAC becomes computationally prohibitive for large-scale instances~\cite{maunu2019robust}. This high computational cost is expected, as RSR under the adversarial contamination model is known to be NP-hard~\cite{haastad2001some, hardt2013algorithms}.
 To mitigate these computational costs, more recent approaches trade robustness for scalability. Specifically, they rely---either explicitly or implicitly---on additional assumptions about the distribution of outliers, thereby making them ineffective under truly adversarial corruption models.

While a detailed discussion of these methods is provided in \Cref{sec:related_works}, we illustrate their failure mechanisms with a toy example. Consider a scenario where the clean samples are drawn from a Normal distribution $N(0, \Sigmastar)$, where the covariance $\Sigmastar \in \mathbb{R}^{d\times d}$ has $\mathrm{rank}(\Sigmastar) = 10$, and the ambient dimension is $d=100$.  An adversary randomly replaces $\epsilon$-fraction of the clean samples with samples independently drawn from $N(0, \widehat{\Sigma})$, where $\mathrm{rank}(\widehat{\Sigma}) = 2$, and the eigenvectors corresponding to nonzero eigenvalues of $\widehat{\Sigma}$ are orthogonal to those of $\Sigmastar$. 

In \Cref{fig:toy_example}, we show the performance of five widely-used RSR algorithms: {\it Tyler's M-estimator} \cite{zhang2016robust}, {\it Fast Median Subspace} \cite{lerman2018fast}, {\it Geodesic Gradient Descent} \cite{maunu2019well}, {\it Randomlized-Find} \cite{hardt2013algorithms}, and RANSAC \cite{maunu2019robust}. A Python implementation of these algorithms (including our proposed method), along with a detailed discussion, is available at:
\begin{center}
\url{https://github.com/Christ1anChen/RSR_under_Adversarial_and_Noisy_Corruptions.git}.
\end{center}

As shown, most existing RSR methods fail rapidly as $\epsilon$ increases, largely due to their reliance on specific assumptions about outliers. In contrast, the classical RANSAC demonstrates greater stability under this adversarial setting. Moreover, despite the desirable robustness of RANSAC to adversarial corruption, its performance is hindered by three major challenges. First, it requires the exact value of the dimension $\rstar$ as input, which is rarely known in practice. \Cref{fig:RANSAC_compare} (left) demonstrates the performance of RANSAC when the dimension $r$ is overestimated by just one, i.e., $r=\rstar+1$. As shown, even this slight overestimation causes RANSAC to fail drastically, highlighting its extreme sensitivity to the choice of $r$. Second, as shown in \Cref{fig:RANSAC_compare} (middle), the addition of Gaussian noise to the samples significantly deteriorates the performance of RANSAC. Finally, as demonstrated in \Cref{fig:RANSAC_compare} (right), while RANSAC is computationally efficient for small values of $\rstar$, it becomes prohibitively expensive to run as $\rstar$ increases, resulting in a 1000-fold slow-down when $\rstar$ is raised from $10$ to $40$.

\begin{figure}
\centering
\includegraphics[width=10cm]{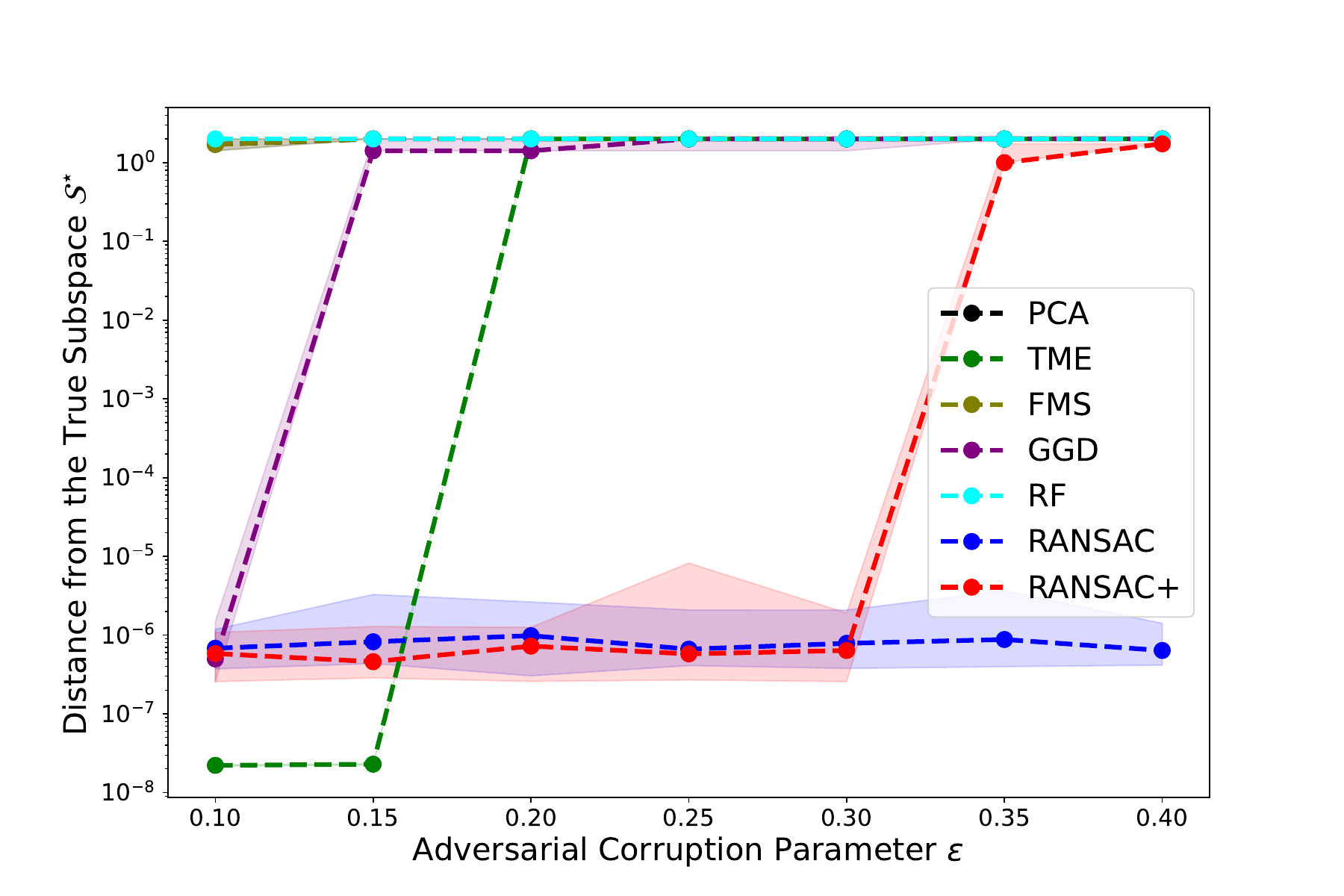}
\caption{The performance of various RSR methods across different corruption levels $\epsilon$. The considered methods are Tyler's M-estimator (TME) \cite{zhang2016robust}, Fast Median Subspace (FMS) \cite{lerman2018fast}, Geodesic Gradient Descent (GGD) \cite{maunu2019well}, Randomized-Find (RF) \cite{hardt2013algorithms}, and the classic RANSAC algorithm \cite{maunu2019robust}. The clean samples are drawn from $N(0, \Sigmastar)$ with $\mathrm{rank}(\Sigmastar) = 10$, while outliers are drawn from $N(0, \widehat{\Sigma})$ with $\mathrm{rank}(\widehat{\Sigma}) = 2$. The Gaussian noise covariance is set to zero in these experiments. The subspace spanned by the outliers are chosen to be orthogonal to $\Sstar$. All nonzero eigenvalues of $\Sigmastar$ are set to $1$, and the nonzero eigenvalues of $\widehat{\Sigma}$ are set to $10$. In all cases, the ambient dimension is set to $d = 100$ and the total sample size to $n=500$.}
\label{fig:toy_example}
\end{figure}

\begin{figure}
\centering
\begin{minipage}[t]{0.32\textwidth} \centering
\includegraphics[width=5.9cm]{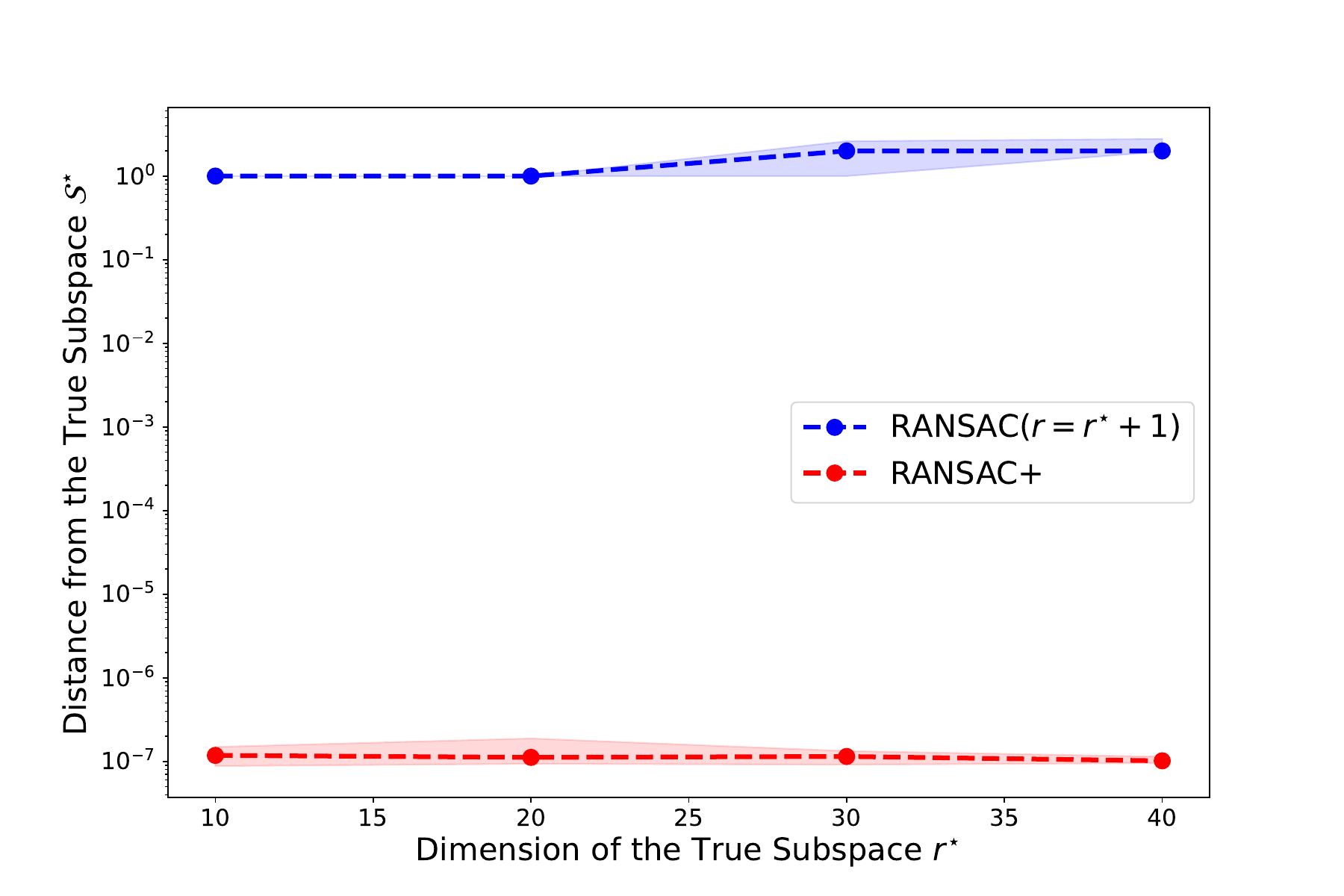}
\end{minipage}
\begin{minipage}[t]{0.32\textwidth} \centering
\includegraphics[width=5.9cm]{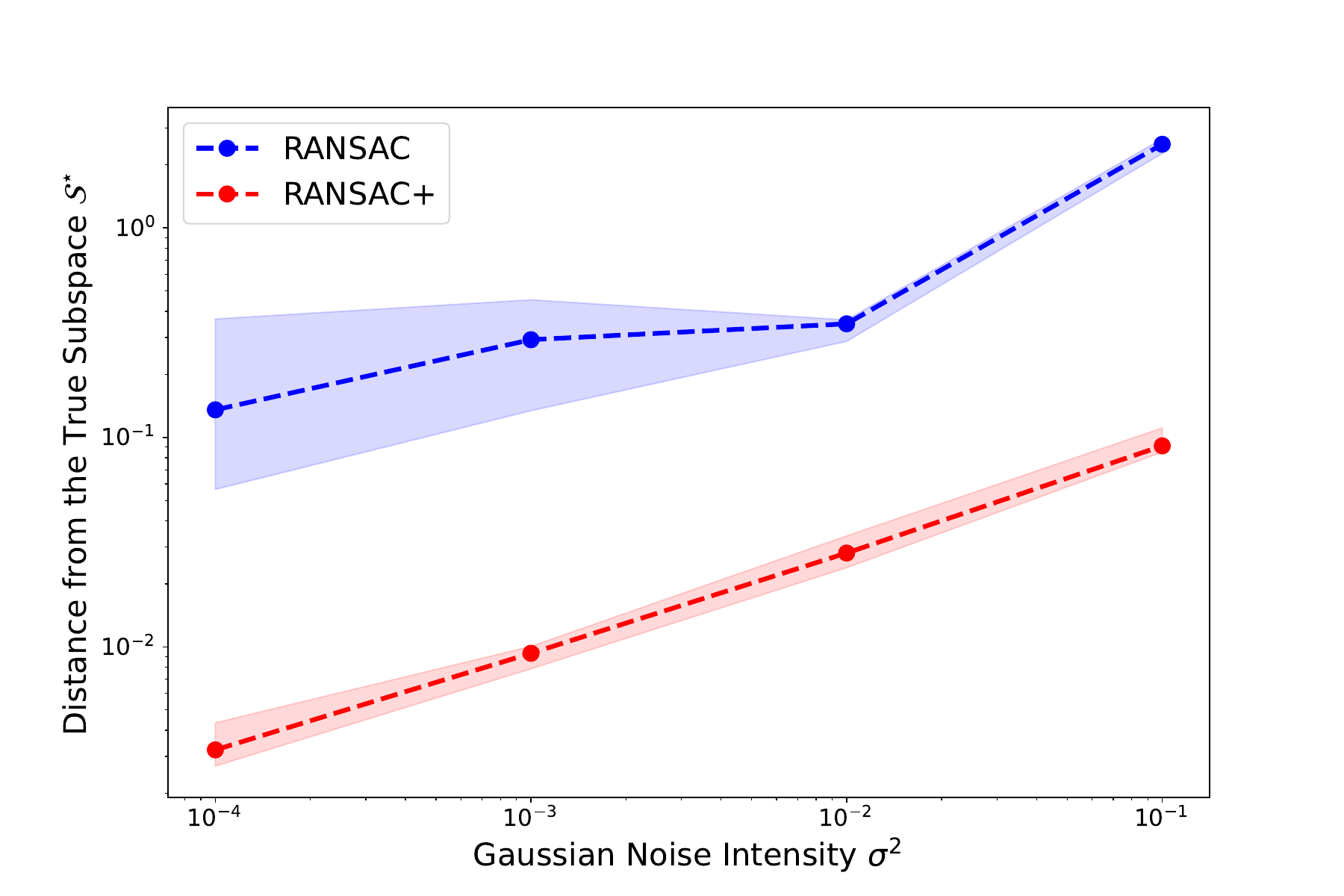}
\end{minipage}
\begin{minipage}[t]{0.32\textwidth} \centering
\includegraphics[width=5.9cm]{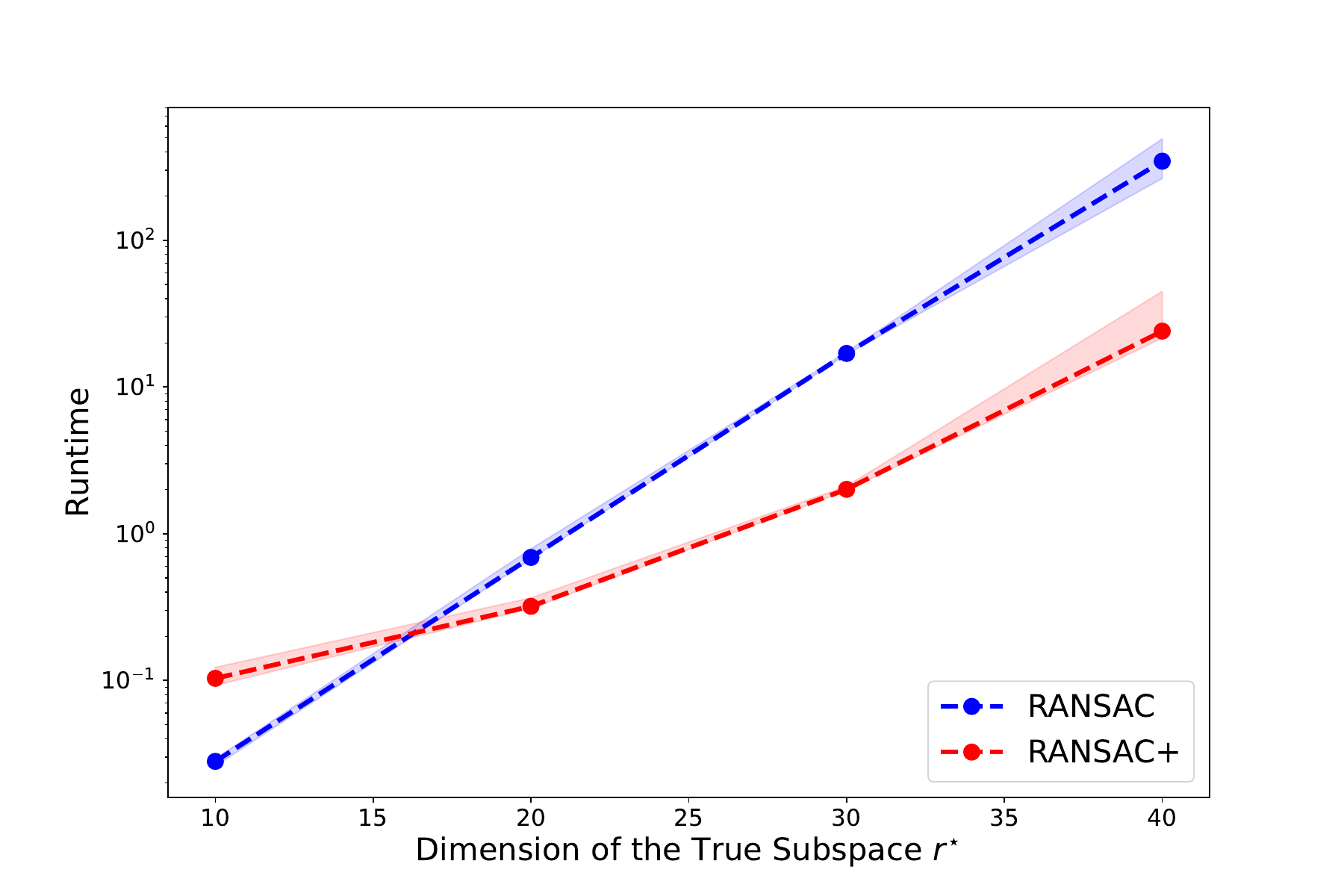}
\end{minipage}
\caption{{\bf (Left)} Performance comparison of RANSAC and RANSAC+ across varying subspace dimensions $\rstar$. The search dimension $r$ for RANSAC is overestimated by one, while RANSAC+ operates without any prior knowledge of $\rstar$. {\bf (Middle)} Performance of RANSAC and RANSAC+ under Gaussian noise with varying variance $\sigma^2 / d$. {\bf (Right)} Runtime of RANSAC (with an exact prior knowledge of $\rstar$) and RANSAC+ for different subspace dimensions $\rstar$. The data generation process follows that of Figure~\ref{fig:toy_example}. In all cases, the total sample size to $n = 500$ and the adversarial corruption parameter to $\epsilon = 0.2$. For the left and middle plots, the ambient dimension is $d=100$, while for the right figure, it is $d=1000$.}
\label{fig:RANSAC_compare}
\end{figure}

\subsection{Summary of contributions}

In this work, we revisit the classical RANSAC algorithm, precisely remedying the scenarios where it breaks down. In particular, our goal is to leverage the robustness of RANSAC against adversarial corruptions while mitigating its sensitivity to Gaussian noise, rank misspecification, and its prohibitive computational cost.

Our proposed approach, which we call {\it two-stage RANSAC} (RANSAC+), consists of two key stages.
In the first stage, we perform an efficient {coarse-grained} estimation of the underlying subspace, yielding an initial estimate of $\Sstar$, denoted as $\cV$. While $\cV$ may not exactly coincide with $\Sstar$, it satisfies two important properties: (1) it nearly contains $\Sstar$ with an error scaling with the noise level, and (2) the dimension $\hat r$ of $\cV$ matches the true dimension $\rstar$ up to a constant factor. Despite its promising performance, our proposed algorithm does not require any prior knowledge of \( r^\star \) and remains highly efficient. Specifically, it requires a sample size that scales near-linearly with \( r^\star \), while its runtime scales linearly with the sample size \( n \) and the ambient dimension \( d \), and near-linearly with \( r^\star \).
An immediate consequence of this result is that the entire sample set can be safely projected onto the subspace $\cV$, effectively reducing the ambient dimension from $d$ to $\hat r$.

In the second stage, we refine the subspace \( \cV\) by applying a robustified variant of the classical RANSAC algorithm for finer-grained estimation. A key advantage of this refinement is that, by operating on the projected sample set within the low-dimensional subspace \(\cV \) identified in the first stage, it enables a provably more accurate and efficient characterization of \( \mathcal{S}^\star \) compared to a direct application of RANSAC.

\Cref{fig:toy_example} illustrates the performance of RANSAC+ compared to the previously discussed RSR methods. As shown, RANSAC+ demonstrates superior robustness to adversarial corruption relative to the other methods, with the exception of classical RANSAC. In contrast, \Cref{fig:RANSAC_compare} highlights how RANSAC+ outperforms RANSAC under dimension misspecification, increased noise levels, and in terms of runtime.

\subsection{Related works}
\label{sec:related_works}
In this section, we review relevant works on RSR, with a particular focus on adversarial robustness. A more comprehensive literature review is provided in 
\cite{lerman2018overview}.

\paragraph{RANSAC-type methods}
The RANSAC algorithm, first introduced in \cite{fischler1981random}, operates as follows: it begins by randomly sampling points until they span a $\rstar$-dimensional subspace. It then determines the number of points whose distances to this subspace fall below a specified threshold, classifying them as clean samples. If the number of clean samples exceeds a predefined consensus threshold, the algorithm outputs the corresponding subspace. Otherwise, after a fixed number of iterations, it returns the subspace with the highest consensus count. 

As empirically demonstrated above, RANSAC suffers from dependence on the true dimension \( r^\star \), sensitivity to noise, and high computational cost. 
To address computational challenges of RANSAC, \cite{hardt2013algorithms} proposed the {\it Randomized-Find} (RF) algorithm, a faster variant of RANSAC. The RF algorithm relies on the key assumption that outliers are in general position---that is, they do not lie within a low-dimensional subspace. For a dataset \( \cX \subseteq \mathbb{R}^d \) with size \( n > \rstar \), RF repeatedly selects random subsets \( \widetilde{\cX} \) of size \( d \) from \( \cX \). It then identifies subsets where \( \text{rank}(\widetilde{\cX}) < d \), as these must contain at least \( \rstar+1 \) clean samples. The indices of the clean samples are then determined from the nonzero elements of a vector in the kernel of \( \widetilde{\cX} \). 
However, RF is not robust against truly adversarial corruptions, as these corruptions typically violate the general location assumption (see Section~\ref{subsec::failure}).
Moreover, extending RF to noisy settings requires unrealistic assumptions, such as the determinant separation condition \cite[Condition 11]{hardt2013algorithms}, which does not align with the adversarial contamination model. 

Inspired by RF, \cite{arias2017ransac} introduced an even more efficient variant of RANSAC that only samples subsets of \( \rstar+1 \) points until a linearly dependent set is identified. However, this method also fails under truly adversarial settings. For instance, if the outliers lie on a one-dimensional subspace, subsets containing two outliers can lead to arbitrarily poor subspace estimates.

\paragraph{Other RSR methods} 
Byond RANSAC-type algorithms, another approach that has shown great promise is Tyler's M-estimator (TME). Originally designed for robust covariance estimation, TME was later shown in \cite{zhang2016robust} to be applicable to RSR. However, this approach also assumes that the outliers are in general positions, thereby making it ineffective against adversarial contamination. TME has notable advantages, such as scale invariance and simplicity of implementation. However, the extension of TME to the noisy RSR setting, as presented in Theorem 3.1 of \cite{zhang2016robust}, is limited. 

Another line of research involves direct optimization of the robust least absolute deviations function over the non-convex Grassmannian manifold $G(d,\rstar)$. The work \cite{lerman2018fast} proposed a {\it fast median subspace} (FMS) algorithm, which employs iteratively re-weighted least squares (IRLS) to minimize an energy function. In a follow-up study, \cite{maunu2019well} developed a geodesic gradient descent (GGD) method to optimize the same energy function on $G(d,\rstar)$, leveraging ideas from \cite{edelman1998geometry}. Additionally, \cite{maunu2019well} analyzed the landscape of this energy function over $G(d,\rstar)$. However, as noted in \cite{maunu2019robust}, this energy function has limitations—it is not robust against adversarial outliers with large magnitudes. 
Despite stronger theoretical guarantees for GGD compared to FMS, neither method effectively handles real adversarial corruption, as previously mentioned.

Finally, another line of research focuses on \textit{robust covariance estimation} by directly minimizing an $\ell_1$-loss~\cite{ma2023global, ding2021rank, ma2021sign, ma2023can}. While accurately estimating the covariance matrix is sufficient to recover the desired subspace, the assumptions required for outliers and inliers in these approaches are significantly more restrictive than those considered in this work.

\subsection{Notation}
We use uppercase letters for matrices, lowercase letters for column vectors, and calligraphic letters for sets. We denote the $d$-dimensional all-zero column vector by $0_d$ and the $d\times d$ identity matrix by $I_d$. For a matrix $A$, $\textbf{col}(A)$ denotes the subspace spanned by the columns of $A$, $\sigma_i(A)$ denotes the $i^{\text{th}}$ largest singular value of $A$, and, if $A$ is square and symmetric, then $\gamma_j(A)$ denotes the $i^{\text{th}}$ largest eigenvalue of $A$. Moreover maximum and minimum singular values are denoted by $\sigma_{\max}(A)$ and $\sigma_{\min}(A)$, respectively. Similarly, the maximum and minimum eigenvalues are defined as $\gamma_{\max}(A)$ and $\gamma_{\min}(A)$, respectively. Depending on the context, $\norm{\cdot}$ denotes the spectral norm of a matrix or the Euclidean norm of a vector. For $r\leq d$, the set of semi-orthogonal $d\times r$ matrices is denoted by $O(d,r) = \{V\in \RR^{d\times r}: V^\top V = I_d\}$. For any subspace $\cS \subseteq \RR^d$, the projection operator onto $\cS$ is denoted as $\mathrm{P}_\cS = \mathrm{P}_U = UU^\top$, where the columns of $U\in O(d, \dim(\cS))$ form an orthonormal basis of $\cS$ in $\RR^d$.  The unit ball of dimension $d$ is denoted by $\mathbb{S}^{d-1}$. The cardinality of a set $\cX$ is denoted by $|\cX|$. The median of a set $\cX\subset \RR$ is denoted by $\mathrm{median}(\cX)$. We define $[n]:=\{1,2,\dots, n\}$. Throughout the paper, the constants $C, C', c_1, c_2, \dots > 0$ denote universal constants. 

\section{Main results}
\label{sec:main}

Recall that, given an $(\epsilon, \Sigma_\xi)$-corrupted sample set obtained according to \Cref{dfn:adversarial_contamination} from a clean distribution $\cP$ that satisfies \Cref{dfn:inliers}, our goal is to recover the subspace spanned by the clean samples, up to an error that is controlled by the noise level of the inliers. 
Let $\Sigmastar = \Ustar \Dstar {\Ustar}^\top$ be the eigen-decomposition of the {\it true covariance} of the clean distribution, where $\Dstar\in \mathbb{R}^{\rstar\times \rstar}$ is a diagonal matrix containing the nonzero eigenvalues of $\Sigmastar$, denoted by $\gammastar_{\max}:=\gammastar_1\geq \dots\geq \gammastar_{\rstar}:=\gammastar_{\min}>0$, and $\Ustar\in O(d,\rstar)$ is the matrix of corresponding eigenvectors.
Given \Cref{dfn:inliers}, the true low-dimensional subspace $\cS^\star$ containing the clean samples coincides with the subspace spanned by the eigenvectors in $\Ustar$, i.e., $\textbf{col}(\Ustar) = \Sstar$. Therefore, it suffices to recover $\Ustar$, as it forms a valid basis for $\Sstar$.

Next, we present our two-stage algorithm in \Cref{alg:RANSAC+}. We defer discussion on different components of this algorithm to \Cref{sec::algorithm}.

\begin{algorithm}[H]
\caption{Two-stage RANSAC (RANSAC+)}
\label{alg:RANSAC+}
\begin{algorithmic}[1]
\Statex {{\bf Require:} The $(\epsilon, \Sigma_\xi)$-corrupted sample set $\cX= \{x_1, \ldots, x_n\}$.}
\State \textbf{Coarse-Grained Estimation:} Use $\cX$ to obtain an orthonormal basis $V$ for a subspace $\cV$ of dimension $\hat{r} \geq \rstar$, approximately containing $\Sstar$, using \Cref{alg:first_stage}.
\State \textbf{Projection:} Project the data points onto the reduced subspace $\cV$, obtaining $\widehat{\cX} = \left\{V^\top x_1, \ldots, V^\top x_n\right\} \subset \mathbb{R}^{\hat{r}}$.
\State \textbf{Fine-Grained Estimation:} Use $\widehat{\cX}$ to obtain an approximate orthonormal basis $\widehat{U}$ for $\Sstar$ with dimension $\rstar$, using \Cref{alg:second_stage}.
\end{algorithmic}
\end{algorithm}

Our first result demonstrates that it is possible to efficiently obtain a coarse estimate of $\Sstar$, i.e., a subspace $\cV$ that approximately contains $\Sstar$ and has a dimension of at most $\cO(\rstar)$.

\begin{theorem}
\label{thm:alg_1}
Given an $(\epsilon, \Sigma_\xi)$-corrupted sample set with a sample size $n = \Omega(\rstar \log (\rstar))$ and Gaussian noise covariance satisfying $\sqrt{\rstar \norm{\Sigma_\xi}} + \sqrt{\tr(\Sigma_\xi)} = \cO\left(\gammastar_{\min}\right)$, there exists an algorithm (\Cref{alg:first_stage}) that runs in $ \cO(nd\rstar \log (\rstar))$ time, and with an overwhelming probability, generates a subspace $\cV$ with dimension $\hat r = \cO(\rstar)$ such that, for each $x \sim \cP$:
\begin{align*}
\norm{\left(I_{d} - \mathrm{P}_{\cV}\right) x} \leq \cO\left(\left( \sqrt{\frac{\rstar\norm{{\Sigma_\xi}}}{\gammastar_{\min}}} + \sqrt{\frac{\tr(\Sigma_\xi)}{\gammastar_{\min}}}\right) \norm{x} \right).
\end{align*}
\end{theorem}

The above result characterizes the performance of our proposed first-stage algorithm, formally introduced in \Cref{alg:first_stage}. A detailed discussion on the algorithm and its analysis are provided in \Cref{sec:first_stage}. Specifically, \Cref{thm:alg_1} demonstrates that \Cref{alg:first_stage} identifies a coarse-grained subspace \(\cV\) of dimension \(\mathcal{O}(\rstar)\) that nearly contains the true subspace \(\Sstar\), with an error proportional to the intensity of the Gaussian noise. As a special case, if the contamination model is purely adversarial ($\Sigma_\xi = 0$), the above theorem implies that $\Sstar\subseteq \cV$ with an overwhelming probability.  Remarkably, the sample complexity required for \Cref{alg:first_stage} scales only with the (unknown) dimension of the true subspace \( \mathcal{S}^\star \), which is optimal up to a logarithmic factor. Moreover, the computational complexity of the first-stage algorithm is \( \mathcal{O}(nd r^\star \log (r^\star)) \), which matches that of a single PCA computation, up to logarithmic factors. 

\begin{figure}[H]
    \centering
    \includegraphics[width=0.85\linewidth]{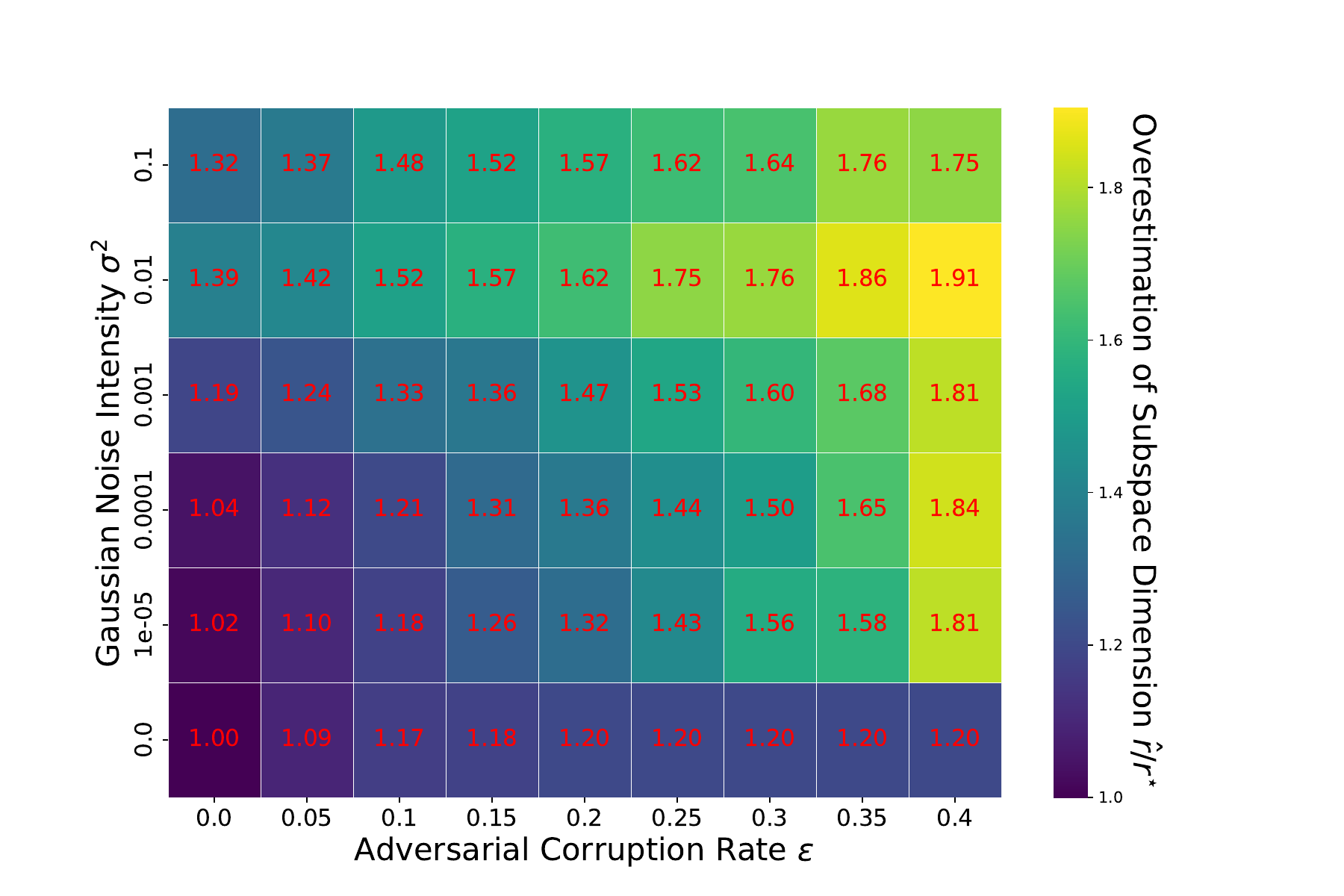}
    \caption{Overestimation of the subspace dimension, measured by \( \hat{r} / \rstar \), for \Cref{alg:first_stage} under varying corruption rates \( \epsilon \) and Gaussian noise levels \( \sigma^2 \). The experimental setup follows the same example described in \Cref{subsec::failure}. For each pair of values \( (\epsilon, \sigma^2) \), we report the average of \( \hat{r} / \rstar \) over 20 independent trials.}
    \label{fig:first_stage}
\end{figure}

Indeed, our proposed first-stage algorithm already provides an efficient method for reliably recovering a subspace that approximately contains the true low-dimensional subspace \( \mathcal{S}^\star \), with its dimension differing from \( \rstar \) by at most a constant factor. Consider the same toy example discussed in \Cref{subsec::failure}. \Cref{fig:first_stage} illustrates the extent of overestimation in the subspace dimension output by the first-stage algorithm, measured as the ratio \( \hat{r}/\rstar \), for varying values of the corruption parameter \( \epsilon \) and the noise variance \( \sigma^2 \). As shown, while increasing both \( \epsilon \) and \( \sigma^2 \) increases the value of \( \hat{r} / \rstar \), the ratio remains upper-bounded by \( 1.91 \), indicating that the overestimated dimension is at most twice the true dimension.
 
 In the following theorem, we show that it is possible to further improve the recovery accuracy and precisely identify the true dimension \( r^\star \) using a finer-grained method.

\begin{theorem}
\label{thm:alg_2}
Suppose that the conditions of \Cref{thm:alg_1} are satisfied and \Cref{alg:first_stage} succeeds. Given the projected sample set $\widehat{\cX} = \left\{V^\top x_1, \ldots, V^\top x_n\right\} \subset \mathbb{R}^{\hat{r}}$, there exists an algorithm (\Cref{alg:second_stage}) that runs in $\cO({\rstar}^3 e^{\rstar})$ time and, with an overwhelming probability, generates a subspace $\widehat{\cS}$ with dimension $\rstar$ that satisfies
\begin{align*}
\norm{\mathrm{P}_{\widehat{\cS}} -\mathrm{P}_{\Sstar} } \leq \cO\left( \sqrt{\frac{\rstar \norm{{\Sigma_\xi}}}{\gammastar_{\min}}} + \sqrt{\frac{\tr(\Sigma_\xi)}{\gammastar_{\min}}}\right).
\end{align*}
\end{theorem}

By combining the results of \Cref{thm:alg_1} and \Cref{thm:alg_2} with the procedure described in \Cref{alg:RANSAC+}, we conclude that RANSAC+ recovers an accurate estimate of the true subspace $\Sstar$ with the correct dimensionality, while keeping a computational complexity of
\begin{align*}
    \cO\Big(\underbrace{d{\rstar}^2+nd\rstar}_{\text{\Cref{alg:first_stage}}} + \underbrace{nd\rstar}_{\text{projection}}+\underbrace{{\rstar}^3 e^{\rstar}}_{\text{\Cref{alg:second_stage}}}\Big).
\end{align*}
This makes RANSAC+ more efficient than the classic RANSAC algorithm, which incurs a computational cost of \(\mathcal{O}(nd\rstar e^{\rstar})\) \cite{maunu2019robust}. We note that, due to the intrinsic NP-hardness of the RSR problem, it is fundamentally impossible to reduce the dependency of our algorithm on $\rstar$ to a purely polynomial form. Furthermore, our result provides a concrete characterization of the error rate for noisy subspace recovery, providing a substantial improvement over prior works on RANSAC and its variants.

\section{Preliminaries}
\label{sec:asp}
Next, we present the necessary tools for our theoretical analysis.
First, we introduce a classic concentration result on the singular values of a Gaussian ensemble.
\begin{lemma}[Theorem 6.1 in \cite{wainwright2019high}]
\label{thm:bound_singular_value}
Let $X = \begin{bmatrix}
    x_1& x_2& \cdots& x_n
\end{bmatrix} \in \RR^{d \times n}$ be drawn from a $\Sigma$-Gaussian ensemble, meaning that each column $x_i$ is independently drawn from a Gaussian distribution with a zero mean and covariance matrix $\Sigma$. Then, for any $t >0$, we have
\begin{align*}
{\Pr}\left( \frac{\sigma_{\max}(X)}{\sqrt{n}} \geq \gamma_{\max}(\Sigma^{1/2})(1+t) + \sqrt{\frac{\tr(\Sigma)}{n}} \right) \leq e^{-nt^2/2}.
\end{align*}
Moreover, when $n\geq d$, we have
\begin{align*}
{\Pr}\left( \frac{\sigma_{\min}(X)}{\sqrt{n}} \leq \gamma_{\min}(\Sigma^{1/2})(1-t) - \sqrt{\frac{\tr(\Sigma)}{n}} \right) \leq e^{-nt^2/2}.
\end{align*}
\end{lemma}

Next, we present tail bounds on binomial and hypergeometric distributions, both of which are direct consequences of Hoeffding's inequality.
\begin{lemma}[Proposition 2.5 in \cite{wainwright2019high}]
\label{thm:median}
Let $z_1, \ldots z_n$ be independently drawn from a Bernoulli distribution with a success probability $p > a$, where $a \in (0,1)$ is a constant. Then, we have
\begin{align*}
{\Pr} \left(\frac{1}{n}\sum_{i=1}^n z_i \leq a\right) \leq  e^{ -2n\left(p-a \right)^2 }.
\end{align*}
\end{lemma}
\begin{lemma}[\cite{chvatal1979tail}]
\label{lem:hypergeometric}
Suppose that $z$ follows a hypergeometric distribution with parameters $(n, k, m)$, where $n$ is the population size, $k$ is the number of success states in the population and $m$ is the number of draws. Then, for any $0\leq t \leq \frac{k}{n}$, we have
\begin{align*}
{\Pr}\left( z \leq \left(\frac{k}{n} - t \right) m \right) \leq  e^{ -2mt^2 }.
\end{align*}
\end{lemma}

Our next assumption entails that a clean sample $x\sim \cP$ is not concentrated too much around zero, when projected along any direction $v\in \col(\Ustar)$. This assumption, famously known as \textit{small-ball condition}, is extensively used in the literature and is crucial in understanding the concentration and distribution properties of random variables, especially in high-dimensional settings \cite{lecue2017sparse, mendelson2015learning, mendelson2017extending, nguyen2013small}. For a random vector $x\sim\cP$ with a covariance $\Sigmastar = \Ustar \Dstar {\Ustar}^\top$, define its normalized variant as $w = {\Dstar}^{-1/2}{\Ustar}^\top x$. Indeed, it is easy to see that $w$ has a zero mean and an identity covariance $I_{\rstar}$. 

\begin{assumption}[Small-ball condition]
\label{asp:anti_concentration}
Given any fixed vector $v\in \RR^{\rstar}$ with $\norm{v} = 1$, random variable $x\sim \cP$, and $w = {\Dstar}^{-1/2}{\Ustar}^\top x$, we have
\begin{align*}
\Pr \left(\left| v^\top w \right| \geq t_0\right) \geq \frac{3}{4},
\end{align*}
where $0<t_0\leq 1$ is a universal constant.
\end{assumption}

The small-ball condition is known to hold for a broad range of distributions (including heavy-tailed ones) and under mild assumptions \cite{mendelson2014remark,mendelson2015learning,lecue2017sparse}. 
For instance, when $w$ has a standard Gaussian distribution, this condition holds with $t_0=\frac{1}{4}$.

Given $n$ i.i.d. copies of $w$, denoted by $\{w_1, \ldots, w_n\}$, our next lemma shows that the minimum eigenvalue of the empirical covariance of $\{w_1, \ldots, w_n\}$ is bounded away from zero with a high probability under \Cref{asp:anti_concentration}.

\begin{sloppypar}
   {\begin{lemma}[Corollary 2.5 in \cite{lecue2017sparse}]
\label{thm:lgd}
Given \Cref{asp:anti_concentration} and i.i.d. samples $x_1,\ldots, x_n\sim \cP$, define $w_i = {\Dstar}^{-1/2} {\Ustar}^\top x_i$ for $i\in [n]$. There exist universal constants $1\!<\!c_1$ and $0\!<\!c_2\!\leq\! 1/2$ such that if the sample size satisfies $n \geq c_1 \rstar$, then with probability at least $1 - e^{-c_2 n}$, we have
\begin{align*}
\gamma_{\min} \left(\frac{1}{n}W_n W_n^\top \right) \geq \frac{3 t_0^2}{8},
\end{align*}
where $W_n = \begin{bmatrix} w_1 & w_2 & \cdots & w_n \end{bmatrix} \in \RR^{\rstar\times n}$.
\end{lemma}} 
\end{sloppypar}

Next, we prove a useful property of the normalized random variables \( \{w_1, \dots, w_n\} \), which will be used in our analysis. Specifically, we establish that, with high probability, most samples in \( \{w_1, \dots, w_n\} \) satisfy the lower bound \( |v^\top w_i| \geq \frac{t_0}{2} \), uniformly over all unit vectors \( v \in \mathbb{S}^{r^\star - 1} \).

\begin{sloppypar}
    \begin{lemma}
\label{thm:anti_concentrate}
Given \Cref{asp:anti_concentration} and i.i.d. samples $x_1,\ldots, x_n\sim \cP$, define $w_i = {\Dstar}^{-1/2} {\Ustar}^\top x_i$ for $i\in [n]$. There exist universal constants $c_3,c_4>0$ such that if the sample size satisfies $n \geq c_3\rstar\log\left(\frac{\rstar}{t_0}\right)$, with probability at least $1-e^{-c_4n}$, we have
\begin{align*}
\inf_{v\in \mathbb{S}^{\rstar-1}} \left|\left\{i\in[n]: |v^\top w_i | \geq \frac{t_0}{2}\right\}\right| \geq \frac{5}{8}n.
\end{align*}
\end{lemma}
\end{sloppypar}
\begin{proof}
By Chebyshev's inequality, for each $i\in [n]$, we have
\begin{align*}
\Pr(\norm{w_i} \geq 4 \sqrt{\rstar}) \leq \frac{\EE[\norm{w_i}^2]}{16\rstar} = \frac{\EE\left[\tr\left(w_i w_i^\top\right)\right]}{16\rstar} = \frac{1}{16}.
\end{align*}
By \Cref{asp:anti_concentration}, for any unit vector $v \in\SS^{\rstar-1}$ and $i\in[n]$, we obtain
\begin{align*}
\Pr\left(|v^\top w_i | \geq t_0,\, \norm{w_i} \leq 4\sqrt{\rstar}\right) \geq \Pr\left(|v^\top w_i | \geq t_0\right) + \Pr\left(\norm{w_i} \leq 4\sqrt{\rstar}\right) -1 \geq \frac{11}{16}.
\end{align*}
Since $\{w_i\}_{i=1}^n$ are independent, for any fixed $v \in\SS^{\rstar-1}$, the cardinality of $\{i\in[n]: |v^\top w_i| \geq t_0,\, \norm{w_i} \leq 4\sqrt{\rstar}\}$ follows a binomial distribution. Thus, given any fixed $v \in\SS^{\rstar-1}$, by \Cref{thm:median}, we obtain
\begin{align*}
\left| \left\{i\in[n]: |v^\top w_i | \geq t_0,\, \norm{w_i} \leq 4\sqrt{\rstar}\right\}\right| \geq \frac{5}{8}n,
\end{align*}
with probability at least $1-e^{-\frac{n}{128}}$. Given any $\epsilon_0>0$, define an $\epsilon_0$-covering net $\cN_{\epsilon_0}$ of the unit sphere $\SS^{\rstar-1}$ as follows \cite[Definition 5.1]{wainwright2019high}:
\begin{align*}
    \cN_{\epsilon_0} = \{v_0: \text{for any $v\in \SS^{\rstar-1}$, there exists $v_0\in \SS^{\rstar-1}$ such that $\norm{v-v_0}\leq \epsilon_0$}\}.
\end{align*}
It is well known that there exists an \( \epsilon_0 \)-covering net of the unit sphere \( \mathbb{S}^{r^\star - 1} \) with cardinality upper bounded by \( |\mathcal{N}_{\epsilon_0}| \leq \left(1 + \frac{2}{\epsilon_0}\right)^{r^\star} \)~\cite[Example 5.8]{wainwright2019high}. A simple union bound on the elements of this \( \epsilon_0 \)-covering net leads to
\begin{align*}
\inf_{v_0\in \cN_{\epsilon_0}} \left|\left\{i\in[n]: |v_0^\top w_i | \geq t_0,\, \norm{w_i} \leq 4\sqrt{\rstar}\right\}\right| \geq \frac{5}{8}n,
\end{align*}
with probability at least $1 - (1+\frac{2}{\epsilon_0})^{\rstar} e^{-\frac{n}{128}}$. On the other hand, for any unit vector $v\in \SS^{\rstar-1}$, there exists $v_0 \in \cN_{\epsilon_0}$ such that $\norm{v_0 - v} \leq \epsilon_0$. Therefore, we have
\begin{align*}
| v^\top w_i |  \geq |v_0^{\top} w_i | - |(v_0 - v)^\top w_i | \geq |v_0^{\top} w_i| - \epsilon_0 \norm{ w_i}.
\end{align*}
Hence, for any $i\in[n]$ such that $|v_0^\top w_i| \geq t_0$ and $\norm{w_i} \leq 4\sqrt{\rstar}$, we obtain $| v^\top w_i |  \geq t_0 - 4\epsilon_0 \sqrt{\rstar}$.
Upon choosing $\epsilon_0 = \frac{t_0}{8\sqrt{\rstar}}$, we can conclude that $| v^\top w_i |  \geq \frac{t_0}{2}$, which implies
\begin{align*}
    &\inf_{v\in \SS^{\rstar-1}}\left|\left\{i\in[n]: |v^\top w_i | \geq \frac{t_0}{2},\, \norm{w_i} \leq 4\sqrt{\rstar}\right\}\right| \\
    &\geq \inf_{v_0\in \cN_{\epsilon_0}}\left|\left\{i\in[n]: |v_0^\top w_i| \geq t_0,\, \norm{w_i} \leq 4\sqrt{\rstar}\right\}\right| \geq \frac{5}{8}n,
\end{align*}
 with probability at least $1 - \left(1+\frac{16\sqrt{\rstar}}{t_0}\right)^{\rstar} e^{-\frac{n}{128}} = 1 - e^{\rstar\log\left(1+\frac{16\sqrt{\rstar}}{t_0}\right) - \frac{n}{128}}.$
Therefore, when $n \geq 256\rstar\log\left(1+\frac{16\sqrt{\rstar}}{t_0}\right)$, with probability at least $1-e^{-\frac{n}{256}}$, we have
\begin{multline*}
    \inf_{v\in \SS^{\rstar-1}}\left|\left\{i\in[n]: |v^\top w_i | \geq \frac{t_0}{2}\right\}\right| \!\geq\! \inf_{v\in \SS^{\rstar-1}}\left|\left\{i\in[n]: |v^\top w_i | \!\geq\! \frac{t_0}{2},\, \norm{w_i} \leq 4\sqrt{\rstar}\right\}\right| \geq \frac{5}{8}n.
\end{multline*}
\end{proof}

\section{Proposed Algorithm}\label{sec::algorithm}

In this section, we formally introduce our two-stage algorithm (RANSAC+) and provide the intuition behind it.

\begin{sloppypar}
    As previously noted, the computational cost of classic RANSAC, given by \( \mathcal{O}\left( nd\rstar e^{\rstar}\right)\), becomes prohibitive in high-dimensional settings when \(d\) or \(n\) is large, even for moderate values of \(\rstar\). While the exponential dependence on \(\rstar\) appears unavoidable, we demonstrate that the dependency on \(d\) and \(n\) can be reduced from multiplicative to additive. The reduction in $d$-dependence is straightforward: by first reducing the ambient dimension from \(d\) to \(\mathcal{O}(\rstar)\), we defer the exponential cost \( e^{\rstar} \) to a lower-dimensional space. The reduction in $n$-dependence arises from a fundamental difference in how our method and classic RANSAC recover the target subspace. Classic RANSAC evaluates each candidate subspace by checking all samples to count how many lie within a specified angular threshold—an operation linear in $n$. In contrast, our method processes candidate batches by computing their spectrum and then recovers the final subspace by identifying a consistent eigengap across batches. Thanks to the projection step, computing the spectrum of each batch depends only on $\rstar$, not on $d$ or $n$.
\end{sloppypar}

More specifically, the first-stage algorithm (\Cref{alg:first_stage}) proceeds as follows. It takes as input an $(\epsilon, \Sigma_\xi)$-corrupted sample set $\cX$, along with an estimate of $\tr(\Sigma_\xi)$ and $\norm{\Sigma_\xi}$. The algorithm begins by initializing the search batch size to $B = 2$. In each iteration, the algorithm randomly samples a subset \(\mathcal{X}_B\) of size \(B\) from \(\mathcal{X}\) and computes the basis $V$ of the subspace $\col(X_B)$. It then calculates the median residual, defined as the median of the distances $\left\{\norm{x - VV^\top x}: x \in \mathcal{X} \setminus \mathcal{X}_B\right\}$. If this median residual falls below the threshold $\eta_{\mathrm{thresh}}$, the algorithm terminates and outputs the corresponding subspace. Otherwise, it doubles the batch size $B$ and continues to the next iteration.

\begin{algorithm}
\caption{First Stage: Course-grained Estimation}
\label{alg:first_stage}
\begin{algorithmic}[1]
\Statex{{\bf Require:} The $(\epsilon, \Sigma_\xi)$-corrupted sample set $\cX= \{x_1, \ldots, x_n\}$, $\tr(\Sigma_\xi)$, and $\norm{\Sigma_\xi}$}.
\State{$B \leftarrow 2$ and $\eta_{\mathrm{thresh}}\leftarrow C\cdot\frac{5\sqrt{\rstar \norm{\Sigma_\xi}} + \sqrt{\tr(\Sigma_\xi)}}{t_0}$ for a sufficiently large constant $C>0$.}
\While{$B < d$}
\State{Randomly select $B$ samples from \(\mathcal{X}\) and collect them into a matrix \(X_B \in \mathbb{R}^{d \times B}\)}.
\State{Compute the basis \( V \in \mathbb{R}^{d \times B} \) for \(\textbf{col}(X_B)\).}
\State{$\mathbf{MedRes}(V)\leftarrow\median\left(\left\{\norm{x - V V^\top x}:x\in \cX \backslash \cX_B\right\}\right)$.}\label{alg:alg1-medres}
\If{$\mathbf{MedRes}(V) \leq \eta_{\mathrm{thresh}}$}\label{alg:alg1-ifthen}
\State{\textbf{break}}
\Else
\State{$B \leftarrow 2B$.\label{alg:alg1-stepsize}}
\EndIf
\EndWhile
\State{Set $\hat r = B$, and subspace $\cV = \textbf{col}(X_B)$.}
\State\Return $V \in O(d, \hat r)$ and $\cV$.
\end{algorithmic}
\end{algorithm}

Next, we outline the main intuition behind the correctness of \Cref{alg:first_stage}, deferring its rigorous analysis to \Cref{sec:first_stage}. When the batch size is underestimated, i.e., \( B < \rstar \), the sampled subspace \(\col(X_B)\) cannot fully capture the true subspace \(\Sstar\). As a result, a significant portion of inliers will be far from \(\col(X_B)\), with a distance proportional to the smallest nonzero eigenvalue of \(\Sigmastar\). For sufficiently small $\tr(\Sigma_\xi)$ and $\norm{\Sigma_\xi}$, this implies that the median of \( \left\{ \norm{x - V V^\top x} : x \in \mathcal{X} \setminus \mathcal{X}_B \right\} \) will exceed the prespecified threshold, prompting the algorithm to double the batch size.
Conversely, when the batch size is overestimated, i.e., \( B \geq c\cdot\rstar \) for some sufficiently large constant \( c > 1 \), the sampled subspace \(\col(X_B)\) will almost entirely contain \(\Sstar\), as it is spanned by at least \( \rstar \) inliers with high probability. Given that the corruption parameter satisfies \( \epsilon < 1/2 \), more than half of the samples will have a distance from the sampled subspace proportional to the noise level. This ensures that the median of \( \left\{ \norm{x - V V^\top x} : x \in \mathcal{X} \setminus \mathcal{X}_B \right\} \) meets the prespecified threshold. 

Once the coarse-grained subspace $\cV$ is obtained from \Cref{alg:first_stage}, one can safely project the original dataset onto this subspace, obtaining $\widehat{\cX} = \left\{V^\top x_1, \ldots, V^\top x_n\right\}\subset \mathbb{R}^{\hat r}$.

\begin{algorithm}
\caption{Second Stage: Fine-grained Estimation} 
\label{alg:second_stage}
\begin{algorithmic}[1]
\State{{\bf Rquire:} Failure probability $\delta$, corruption parameter $\epsilon$, projected sample set $\widehat{\cX} = \left\{V^\top x_1, \ldots, V^\top x_n\right\}\subset \mathbb{R}^{\hat r}$, basis matrix $V\in O(d, \hat r)$ produced by \Cref{alg:first_stage}.}
\State{$B\leftarrow C'\cdot\max\left\{\hat r, \log\left(\frac{3}{\delta}\log\left(\frac{1}{\delta}\right)\right)\right\}$ for sufficiently large constant $C'>0$.}\label{line::B-second-stage}
\State{$T\leftarrow \left(\frac{1}{1-1.1\epsilon}\right)^B\log\left(\frac{1}{\delta}\right)$.}\label{line::T-second-stage}
\While{$j < T$}
\State{Randomly select $B$ samples from $\widehat{\cX}$ and collected them into a matrix $\widehat{X}_j\in \RR^{\hat r \times B}$.}\label{alg:alg2-batch}
\State{Compute the singular values of $\widehat{X}_j$ as $\hat{\sigma}_1^{(j)}\geq \hat{\sigma}_2^{(j)}\geq \dots\geq \hat{\sigma}_{\hat r}^{(j)}$.}\label{alg:alg2-singvals}
\EndWhile
\State{Compute $\hat{\sigma}_i := \min_{j\in [T]}\hat{\sigma}_i^{(j)}$, for $i\in[\hat r]$.}
\State{Find the smallest index $\tilde{r}\in[\hat r]$ such that $\hat{\sigma}_{\tilde{r}+1}^2 \leq C'\norm{\Sigma_\xi}$ for sufficiently large $C'>0$.}\label{line::success-second-stage}
\State{Find $k = \argmin_{j\in[T]} \hat{\sigma}_{\tilde{r}+1}^{(j)}$.}\label{line::min-assignment}
\State{Compute the top-$\tilde{r}$ left singular vectors of $\widehat{X}_k \in \RR^{\hat r\times B}$ as $\widehat{U}_k \in \RR^{\hat r\times \tilde{r}}$.}
\State\Return $V \widehat{U}_k \in O(d, \tilde{r})$ and $\widehat{\cS} = \col\left(V \widehat{U}_k\right).$
\end{algorithmic}
\end{algorithm}

The second-stage algorithm (\Cref{alg:second_stage}) takes as input the projected sample set $\widehat{\cX}$ and the basis matrix $V\in O(d, \hat r)$ produced by the first-stage algorithm. It runs for a total of $T$ iterations. In each iteration $j$, the algorithm selects $B = \Omega(\hat r)$ samples uniformly at random and arranges them into a data matrix $\widehat{X}_j\in \RR^{\hat{r}\times B}$. The batch size $B = \Omega(\hat r)$ ensures that, with overwhelming probability, each batch contains at least $\rstar$ inliers. For each batch, the algorithm computes the singular values of $\widehat{X}_j$, denoted as $\hat{\sigma}_1^{(j)}\geq \hat{\sigma}_2^{(j)}\geq \dots\geq \hat{\sigma}_{\hat{r}}^{(j)}$. It then determines the smallest $i^{\text{th}}$ singular value across all batches: $\hat{\sigma}_i=\min_{j\in [T]} \hat{\sigma}^{(j)}_{i}$, for every index $i\in [\hat{r}]$. By appropriately choosing $T$, the algorithm can reliably detect a gap in the sequence $\hat{\sigma}_1\geq \hat{\sigma}_2\geq \dots \geq \hat{\sigma}_{\hat{r}}$. Notably, when $i=\rstar+1$, $\hat{\sigma}_i$ becomes significantly smaller than $\hat{\sigma}_{i-1}$, allowing the algorithm to recover the true subspace dimension $\rstar$ with high probability. Finally, it estimates the subspace by combining the subspace from the first-stage algorithm with the one associated with the batch that has the smallest $(\rstar+1)^{\text{th}}$ singular value.

\subsection{Analysis of the First-stage Algorithm}
\label{sec:first_stage}
In this section, we present the theoretical analysis of the first-stage algorithm (\Cref{alg:first_stage}). In particular, our goal is to prove the following theorem, which is the formal version of \Cref{thm:alg_1}. 
\begin{sloppypar}
    \begin{theorem}\label{thm::first-stage-formal}
    Fix $\rstar\!\geq\! 2$, $\epsilon\!\leq\! 0.1$, $C\!\geq\! 2.2$, $n\!\geq\! \max\left\{2c_3\log \left(\frac{\rstar}{t_0}\right),\! 50c_1,\frac{100}{c_2},\!800\right\}\rstar$, and $n \geq \max\left\{\frac{1}{c_4}, 250\right\} \log(\frac{3}{\delta})$, where $c_1,c_2,c_3,c_4$ are the same universal constants in \Cref{thm:lgd} and \Cref{thm:anti_concentrate}. Moreover, suppose that $5\sqrt{\rstar \norm{\Sigma_\xi}} + \sqrt{\tr(\Sigma_\xi)}\leq \frac{t_0^2}{4t_0+4C}\sqrt{\gammastar_{\min}}$. Then, with probability at least $1-\delta$, \Cref{alg:first_stage} terminates within $\cO(dn{\rstar}\log (\rstar))$ time, and satisfies:
    \begin{align*}
    \hat r &\leq \max\left\{2c_1,\frac{4}{c_2},32\right\}\cdot \rstar,\\
    \norm{\left(I_{d} - VV^\top\right) x} &\leq \frac{4C}{t_0^2}\left(5 \sqrt{\frac{\rstar\norm{{\Sigma_\xi}}}{\gammastar_{\min}}} + \sqrt{\frac{\tr(\Sigma_\xi)}{\gammastar_{\min}}}\right) \norm{x}, \quad \text{for}\quad x\sim \cP.
\end{align*}
\end{theorem}
\end{sloppypar}

Our strategy to prove \Cref{thm::first-stage-formal} proceeds as follows. 
First, we prove that when \( B < \rstar \), the stopping criterion in Line~\ref{alg:alg1-ifthen} is not triggered with high probability, thereby preventing the algorithm from terminating prematurely. Second, we prove that when the stopping criterion in Line~\ref{alg:alg1-ifthen} is triggered, then the recovered subspace $\col(V)$ approximately contains $\Sstar$ with an error proportional to the noise level. Third, we show that when \( \rstar \leq B \leq c\rstar \) for some universal constant \( c > 1 \), the stopping criterion in Line~\ref{alg:alg1-ifthen} is triggered with high probability.

\subsubsection{Step 1: \texorpdfstring{The condition $B < \rstar$ prevents termination}{under-parameterized}}
\label{sec:underparam}

When $B<\rstar$, our next proposition provides a non-vanishing lower bound on the value of $\mathbf{MedRes}(V)$ computed as the median of $\left\{\norm{x - V V^\top x}:x\in \cX \backslash \cX_B\right\}$ in Line \ref{alg:alg1-medres}.

\begin{sloppypar}
\begin{proposition}
\label{prop::phase1}
Fix \( \epsilon \leq 0.1 \) and \( n \geq \max\left\{2c_3\rstar\log\left(\frac{\rstar}{t_0}\right),\,6B\right\} \).
Assuming \( B < \rstar \), with probability at least \( 1 - e^{-c_4 n} - e^{-\frac{n}{250}} \), $\mathbf{MedRes}(V)$ computed in Line \ref{alg:alg1-medres} satisfies:
\begin{align*}
\mathbf{MedRes}(V) \geq \frac{t_0}{2} \sqrt{\gammastar_{\min}} - \left( \sqrt{\tr(\Sigma_\xi)} + 5 \sqrt{\norm{\Sigma_\xi}} \right).
\end{align*}
\end{proposition}
\end{sloppypar}
Before providing the proof of this proposition, we note that, due to our assumed upper bound on the noise level in~\Cref{thm::first-stage-formal}, we have 
\begin{align*}
\frac{t_0}{2} \sqrt{\gammastar_{\min}} - \left( \sqrt{\tr(\Sigma_\xi)} + 5 \sqrt{\norm{\Sigma_\xi}} \right)\geq \frac{t_0}{2} \sqrt{\gammastar_{\min}} - \left( \sqrt{\tr(\Sigma_\xi)} + 5 \sqrt{\rstar\norm{\Sigma_\xi}} \right)> \eta_{\mathrm{thresh}}.
\end{align*}

Hence, if \Cref{prop::phase1} holds, then the stopping criterion of \Cref{alg:first_stage} will not be triggered.
Next, we proceed with the proof of \Cref{prop::phase1}. To this goal, we denote the orthogonal complement of $V$ as $V_\perp \in \RR^{d\times (d-B)}$, and show the following lemma.
\begin{lemma}\label{lem::Vperp}
    When $B<\rstar$, we have $\norm{ V_\perp^\top \Ustar} = 1$.
\end{lemma}
\begin{proof}
     One can write
\begin{align*} 
\dim(\textbf{col}(V_\perp)\cap \textbf{col}(\Ustar)) \!\geq\! \dim(\textbf{col}(V_\perp)) \!+\! \dim(\textbf{col}(\Ustar)) \!-\! d = (d-B)+\rstar-d = \rstar-B \geq 1,
\end{align*}
Thus, there exists a vector $x\in \textbf{col}(V_\perp)\cap \textbf{col}(\Ustar)$ with $\norm{x} =1$ such that $x = V_\perp y_1 = \Ustar y_2,$
for some $y_1 \in \RR^{d-B}$ and $y_2 \in \RR^{\rstar}$. This implies that $1 = \norm{x} = \norm{ V_\perp y_1} = \norm{y_1}$.
Similarly, we have $1 = \norm{y_2}$. On the other hand, $V_\perp y_1 = \Ustar y_2$ leads to $y_1 = V_\perp^\top \Ustar y_2$, implying
\begin{align*}
1 = \norm{ y_1 } =  \norm{ V_\perp^\top \Ustar y_2 } \leq \norm{ V_\perp^\top \Ustar } \norm{y_2 } = \norm{ V_\perp^\top \Ustar}.
\end{align*}
On the other hand, $\norm{ V_\perp^\top \Ustar} \leq \norm{V_\perp^\top} \norm{\Ustar} = 1$.
Therefore, we conclude that $\norm{ V_\perp^\top \Ustar} = 1$.
\end{proof}

Now, we are ready to present the proof of \Cref{prop::phase1}.

\begin{proof}[Proof of \Cref{prop::phase1}]
For any $x\in \RR^d$, we have $\norm{x-VV^\top x} = \norm{V_\perp V_\perp^\top x} = \norm{V_\perp^\top x}$. Therefore, $\mathbf{MedRes}(V)=\median\left(\left\{\norm{V_\perp^\top x}:x\in \cX \backslash \cX_B\right\}\right)$.
Recall that, given any inlier $x$, we have $x = \Ustar {\Dstar}^{1/2}w + \xi$, where $\Ustar {\Dstar}^{1/2}w\sim \cP$ and $\xi \sim N(0_d, \Sigma_\xi)$. It follows that 
\begin{align}\label{eq::LB2}
    \norm{V_\perp^\top x} = \norm{V_\perp^\top (\Ustar {\Dstar}^{1/2}w + \xi)}
\geq \norm{V_\perp^\top \Ustar {\Dstar}^{1/2}w} - \norm{V_\perp^\top  \xi}
\geq \norm{V_\perp^\top \Ustar {\Dstar}^{1/2}w} - \norm{\xi}.
\end{align}
Define the compact SVD of $V_\perp^\top \Ustar$ as $V_\perp^\top \Ustar = P \Lambda Q^\top$, where $P\in O(d-B, \rstar)$, $\Lambda \in \RR^{\rstar\times \rstar}$ is a diagonal matrix with diagonal entries $\lambda_1\geq \dots\geq \lambda_{\rstar}\geq 0$, and $Q \in O(\rstar,\rstar)$. We have
\begin{align*}
\norm{V_\perp^\top \Ustar {\Dstar}^{1/2}w} = \norm{P \Lambda Q^\top {\Dstar}^{1/2}w} = \norm{\sum_{i=1}^{\rstar} \lambda_i \left(q_i^\top {\Dstar}^{1/2}w\right) p_i}=\sqrt{ \sum_{i=1}^{\rstar}\left|\lambda_i \left(q_i^\top {\Dstar}^{1/2}w\right) \right|^2 },
\end{align*}
where $p_i \in\RR^{\rstar}$ and $q_i\in \RR^{\rstar}$ denote the $i^{\text{th}}$ column of $P$ and $Q$, respectively. The last equality is due to the fact that the columns of $P$ are pairwise orthogonal. On the other hand, since $\norm{V_\perp^\top \Ustar} = 1$ according to \Cref{lem::Vperp}, we have $\lambda_1= 1$. This yields:
\begin{align}\label{eq::LB}
    \norm{V_\perp^\top \Ustar {\Dstar}^{1/2}w} \geq \left|\lambda_1 \left(q_1^\top {\Dstar}^{1/2}w\right)\right| =  \left|q_1^\top {\Dstar}^{1/2}w\right|.
\end{align}
Upon defining $v = \frac{{\Dstar}^{1/2} q_1}{\norm{ {\Dstar}^{1/2} q_1 }}$ with $\norm{ v } = 1$, we have
\begin{align*}
\norm{V_\perp^\top \Ustar {\Dstar}^{1/2}w} \geq \norm{{\Dstar}^{1/2} q_1} \, \left|v^\top w\right| \geq \sqrt{\gammastar_{\min}} \, \left|v^\top w\right|.
\end{align*}
Combined with \Cref{eq::LB2}, this implies that
\begin{align}\label{eq::LB3}
    \norm{V_\perp^\top x}\geq \sqrt{\gammastar_{\min}}\left|v^\top w\right|-\norm{\xi}.
\end{align}
Let $\cN$ collect the index set of the samples within the set $\cX\backslash\cX_B$. Moreover, let $\cN_{\mathrm{in}}\subseteq \cN$ denote the index set of the inliers within the set $\cX\backslash\cX_B$. Note that $|\cN| = n-B$ and $|\cN_{\mathrm{in}}|\geq n-\epsilon n-B$.
By \Cref{thm:anti_concentrate}, when $|\cN_{\mathrm{in}}| \geq c_3\rstar\log(\rstar/t_0)$, with probability at least $1-e^{-c_4n}$, we have
\begin{align*}
\left|\left\{i\in \cN_{\mathrm{in}}: \left|v^\top w_i\right| \geq \frac{t_0}{2}\right\}\right| \geq \frac{5}{8} \left|\cN_{\mathrm{in}}\right|.
\end{align*}
Since $\left|\cN_{\mathrm{in}}\right|\geq n-\epsilon n-B$, when $\epsilon \leq 0.1$ and  $B \leq n/6$, we have
\begin{align}\label{eq::LB-count}
    \left|\left\{i\in \cN_{\mathrm{in}}: \left|v^\top w_i\right| \geq \frac{t_0}{2}\right\}\right| \geq \left(\frac{1}{2}+\frac{1}{20}\right)(n - B).
\end{align}
Next, for every $i\in \cN$, define the Bernoulli random variable
\begin{align*}
z_i = \begin{cases}
    1 & \text{if } \norm{ \xi_i } \leq \sqrt{\tr(\Sigma_\xi)} + 5 \sqrt{\norm{\Sigma_\xi}},\\
    0 & \text{otherwise.}
\end{cases}
\end{align*}
Invoking \Cref{thm:bound_singular_value} with $t=4$ implies that $\Pr(z_i=1) \geq 1-e^{-8}$. On the other hand, since $\{z_i\}_{i\in \cN}$ are i.i.d., applying \Cref{thm:median} with $a = 1-\frac{1}{20}$, implies that, with probability at least $1 - e^{-\frac{n-B}{205}} \geq 1 - e^{-\frac{n}{250}}$ (since $B \leq n/6$), we have
\begin{align}\label{eq::LB-count2}
    \left|\left\{i\in \cN: \norm{ \xi_i } \leq \sqrt{\tr(\Sigma_\xi)} + 5 \sqrt{\norm{\Sigma_\xi}}\right\}\right| \geq \left(1-\frac{1}{20}\right)(n - B).
\end{align}
By combining \Cref{eq::LB-count} and \Cref{eq::LB-count2} with \Cref{eq::LB3}, we obtain
\begin{align*}
\left|\left\{i\in \cN: \norm{V_\perp^\top x_i} \geq \frac{t_0}{2} \sqrt{\gammastar_{\min}} - \left(\sqrt{\tr(\Sigma_\xi)} + 5 \sqrt{\norm{\Sigma_\xi}}\right) \right\}\right| \geq \frac{1}{2}(n-B) = \frac{1}{2}|\cN|,
\end{align*}
with probability at least $1 - e^{-c_4n} - e^{\frac{n}{250}}$. Therefore, with the same probability, we have
\begin{align*}
\mathbf{MedRes}(V) = \median\left(\left\{\norm{V_\perp^\top x_i}: i\in \cN\right\}\right)\geq \frac{t_0}{2} \sqrt{\gammastar_{\min}} - \left(\sqrt{\tr(\Sigma_\xi)} + 5 \sqrt{\norm{\Sigma_\xi}}\right).
\end{align*}
\end{proof}

\subsubsection{Step 2: Stopping criterion guarantees correct termination}
Our next proposition shows that as long as the stopping criterion in Line \ref{alg:alg1-ifthen} is satisfied, the first-stage algorithm yields an estimate of $\cV$ with provably bounded error.

\begin{sloppypar}
\begin{proposition}
\label{prop::phase1.5}
Fix \( \epsilon \leq 0.1 \) and \( n \geq \max\left\{2c_3\rstar\log(\rstar/t_0),\,6 B\right\} \). Assume that the stopping criterion in Line \ref{alg:alg1-ifthen} is satisfied, i.e. $\mathbf{MedRes}(V) \leq \frac{C}{t_0} \left(5\sqrt{\rstar \norm{\Sigma_\xi}} + \sqrt{\tr(\Sigma_\xi)} \right)$. Then, with probability at least \( 1 - e^{-c_4 n} - e^{-\frac{n}{250}} \), we have
\begin{align*}
\norm{V_\perp^\top \Ustar {\Dstar}^{1/2}} \leq \frac{4C}{t_0^2}\left(5\sqrt{\rstar\norm{\Sigma_\xi}} + \sqrt{\tr(\Sigma_\xi)}\right).
\end{align*}
\end{proposition}
\end{sloppypar}

\begin{proof}
Suppose by contradiction that 
\begin{align*}
\norm{V_\perp^\top \Ustar {\Dstar}^{1/2}} > \frac{4C}{t_0^2}\left(5\sqrt{\rstar\norm{\Sigma_\xi}} + \sqrt{\tr(\Sigma_\xi)}\right).    
\end{align*} 
From \Cref{eq::LB2}, for any inlier $x$, we obtain
\begin{align*}
\norm{V_\perp^\top x}  = \norm{V_\perp^\top(\Ustar {\Dstar}^{1/2} w + \xi)} \geq \norm{V_\perp^\top \Ustar {\Dstar}^{1/2} w} - \norm{\xi}.
\end{align*}
From \Cref{eq::LB}, we can further obtain
\begin{align*}
\norm{V_\perp^\top \Ustar {\Dstar}^{1/2} w} \geq \norm{V_\perp^\top \Ustar {\Dstar}^{1/2} } |q_1^\top w|,
\end{align*}
where $q_1 \in \RR^{\rstar}$ denotes the right singular vector corresponding to the largest singular value of $V_\perp^\top \Ustar {\Dstar}^{1/2}$. Therefore, we conclude that
\begin{multline}\label{eq::LB_step2}
    \norm{V_\perp^\top x} \geq \norm{V_\perp^\top \Ustar {\Dstar}^{1/2} } |q_1^\top w| - \norm{\xi}
    > \frac{4C}{t_0^2}\left(5\sqrt{\rstar\norm{\Sigma_\xi}} + \sqrt{\tr(\Sigma_\xi)}\right) |q_1^\top w| - \norm{\xi}.
\end{multline}
Identical to the proof of \Cref{prop::phase1} and using the same notation therein, one can verify that the following inequalities hold with probability at least $1 - e^{-c_4n} - e^{\frac{n}{250}}$:
\begin{align*}
    \left|\left\{i\in \cN_{\mathrm{in}}: \left| q_1^\top w_i\right| \geq \frac{t_0}{2}\right\}\right| &\geq \left(\frac{1}{2}+\frac{1}{20}\right)(n - B),\\
     \left|\left\{i\in \cN: \norm{ \xi_i } \leq \sqrt{\tr(\Sigma_\xi)} + 5 \sqrt{\norm{\Sigma_\xi}}\right\}\right| &\geq \left(1-\frac{1}{20}\right)(n - B).
\end{align*}
Combining the above two inequalities with \Cref{eq::LB_step2} and $\frac{C}{t_0} \geq 1$, we obtain
\begin{align*}
\left|\left\{i\in \cN: \norm{V_\perp^\top x_i} > \frac{C}{t_0} \left(5\sqrt{\rstar \norm{\Sigma_\xi}} + \sqrt{\tr(\Sigma_\xi)} \right) \right\}\right| \geq \frac{1}{2}(n-B) = \frac{1}{2}|\cN|.
\end{align*}
This in turn implies that
\begin{align*}
\mathbf{MedRes}(V) = \median\left(\left\{\norm{V_\perp^\top x_i}: i\in \cN\right\}\right) > \frac{C}{t_0} \left(5\sqrt{\rstar \norm{\Sigma_\xi}} + \sqrt{\tr(\Sigma_\xi)} \right),
\end{align*}
with the same probability, thereby arriving at a contradiction.
\end{proof}


\subsubsection{Step 3: \texorpdfstring{The condition $B \geq c\rstar$ guarantees correct termination}{over-parameterized}}
\label{sec:overparam}
Our next proposition shows that, when $B\geq c\rstar$, for some sufficiently large $c\geq 1$, $\mathbf{MedRes}(V)$ is guaranteed to be small with high probability.
\begin{proposition}
\label{prop:residual}
Fix $\rstar \geq 2$, $\epsilon \leq 0.1$, $B \geq \max\{ 2c_1,\, 4/c_2,\, 32\}\cdot \rstar$ and $n\geq 25 B$ with $c_1, c_2$ described in \Cref{thm:lgd}. Then, with probability at least $1 - e^{-\frac{n}{60}}$, $\mathbf{MedRes}(V)$ computed in Line~\ref{alg:alg1-medres} satisfies
\begin{align*}
\mathbf{MedRes}(V)\leq  \left(\frac{6\sqrt{\rstar}}{t_0} + 5\right) \sqrt{\norm{\Sigma_\xi}} + \left(\frac{1}{t_0}+ 1\right)\sqrt{\tr(\Sigma_\xi)}.
\end{align*}
\end{proposition}

Before presenting the proof of \Cref{prop:residual}, we note that, due to the definition of the threshold $\eta_{\mathrm{thresh}}$ in \Cref{alg:first_stage}, and the fact that $0<t_0\leq 1$, $C\geq 2.2$, we must have
\begin{align*}
\left(\frac{6\sqrt{\rstar}}{t_0} + 5\right) \sqrt{\norm{\Sigma_\xi}} + \left(\frac{1}{t_0}+ 1\right)\sqrt{\tr(\Sigma_\xi)}\leq \eta_{\mathrm{thresh}}.
\end{align*}
Therefore, if \Cref{prop:residual} holds, then the stopping criterion of \Cref{alg:first_stage} is triggered.

Without loss of generality, let $\{x_1,x_2,\dots, x_{B_{\rm in}}\}$ be the set of inliers within the batch $\cX_{B}$ in Line \ref{alg:alg1-medres}. To provide the proof of \Cref{prop:residual}, we first provide a lower bound on the number of inliers $B_{\rm in}$ within the batch $\cX_B$.

\begin{lemma}
\label{lem:frac_clean}
For any subset $\cX_B$ sampled uniformly at random without replacement from $\cX$, let $B_{\rm in}$ denote the number of inliers in $\cX_B$. Then, for any $0\leq t\leq 1-\epsilon$, we have
\begin{align*}
\Pr(B_{\rm in} \leq (1-\epsilon - t) B) \leq e^{-2t^2 B}.
\end{align*}
\end{lemma}
\begin{proof}
The proof follows immediately by noting that $B_{\rm in}$  follows a hypergeometric distribution with parameters $(n, (1-\epsilon)n, B)$ and invoking \Cref{lem:hypergeometric}.
\end{proof}

As a direct consequence of \Cref{lem:frac_clean} and noting that $\epsilon\leq 0.1$, we have
\begin{align*}
{\Pr} \left(B_{\rm in} \geq \frac{1}{2} B\right) \geq 1 - e^{2\left(\frac{1}{2} - \epsilon\right)^2 B}.
\end{align*}
Therefore, we conclude that at least half of the samples within a single batch are inliers with high probability. Next, we prove the following key lemma, which controls the estimation error of the subspace $\cV$ recovered by \Cref{alg:first_stage}.

\begin{lemma}
\label{lem:convergence}
Suppose that $B_{\rm in} \geq c_1 \rstar$. With probability at least $1-e^{-c_2 B_{\rm in}}-e^{-\frac{B_{\rm in}}{8}}$, we have
\begin{align*}
\norm{V_\perp^\top \Ustar {\Dstar}^{1/2}} \leq \frac{3 \sqrt{B_{\rm in}\norm{\Sigma_\xi}} + 2\sqrt{\tr(\Sigma_\xi)}}{t_0\sqrt{B_{\rm in}}}.
\end{align*}
\end{lemma}
\begin{proof}
Note that $V_\perp^\top x_i = 0$ for $i\in [B_{\rm in}]$. This implies that
\begin{align*}
\begin{aligned}
\norm{V_\perp^\top \Ustar {\Dstar}^{1/2}} &= \sup_{\norm{v} = 1} \norm{V_\perp^\top \Ustar {\Dstar}^{1/2} v}\\
&= \sup_{\norm{v} = 1} \inf_{\lambda \in \RR^{B_{\rm in}}} \norm{V_\perp^\top \left(\Ustar {\Dstar}^{1/2} v - \sum_{i=1}^{B_{\rm in}} \lambda_i x_i\right)}\\
&\leq \sup_{\norm{v} = 1} \inf_{\lambda \in \RR^{B_{\rm in}}} \norm{ \Ustar {\Dstar}^{1/2} v - \sum_{i=1}^{B_{\rm in}} \lambda_i \left(\Ustar {\Dstar}^{1/2}w_i + \xi_i\right) }\\
&= \sup_{\norm{v} = 1} \inf_{\lambda \in \RR^{B_{\rm in}}} \norm{ \Ustar {\Dstar}^{1/2} \left(v- \sum_{i=1}^{B_{\rm in}} \lambda_i w_i \right) - \sum_{i=1}^{B_{\rm in}} \lambda_i \xi_i }.
\end{aligned}
\end{align*}
Due to the absolute continuity of the probability distribution $\cP$ (\Cref{dfn:inliers}), the normalized i.i.d. samples $\{w_1,\dots, w_{B_{\rm in}}\}$ with $B_{\rm in}\geq \rstar$ span the whole space $\RR^{\rstar}$ almost surely. This implies that the linear system $\sum_{i=1}^{B_{\rm in}} \lambda_i w_i = v$ has a solution $\lambda^\star$. In particular, upon defining $W = \begin{bmatrix} w_1 & w_2 & \cdots & w_{B_{\rm in}} \end{bmatrix} \in \RR^{\rstar\times B_{\rm in}}$, one can take $\lambda^\star =  W^{\dagger} v$, where $W^{\dagger}$ denotes the pseudo-inverse of $W$. Letting $\Xi = \begin{bmatrix}
\xi_1 & \xi_2 & \cdots & \xi_{B_{\rm in}}    
\end{bmatrix} \in \RR^{d\times B_{\rm in}}$, it follows that
\begin{align*}
\begin{aligned}
\norm{ V_\perp^\top \Ustar {\Dstar}^{1/2} } &\leq \sup_{\norm{v} = 1} \norm{\sum_{i=1}^{B_{\rm in}} \lambda^\star_i \xi_i}= \sup_{\norm{v} = 1} \norm{\Xi W^{\dagger} v}\leq \norm{\Xi}\norm{W^{\dagger}}= \frac{\norm{\Xi}}{\sigma_{\min}(W)}.
\end{aligned}
\end{align*}
By \Cref{thm:lgd}, when $B_{\rm in}\geq c_1 \rstar$, we have
\begin{align*}
\sigma_{\min}(W) \geq \sqrt{\frac{3 t_0^2 B_{\rm in}}{8}} \geq \frac{t_0}{2} \sqrt{B_{\rm in}},
\end{align*}
with probability at least $1-e^{-c_2 B_{\rm in}}$. 
Moreover, by \Cref{thm:bound_singular_value} with $t = 1/2$, we have
\begin{align*}
\norm{ \Xi } = \sigma_{\max}(\Xi) \leq \frac{3}{2} \sqrt{B_{\rm in} \norm{\Sigma_\xi}} + \sqrt{\tr(\Sigma_\xi)},
\end{align*}
with probability at least $1-e^{-\frac{B_{\rm in}}{8}}$. By combining the two bounds mentioned above, we obtain:
\begin{align*}
\norm{ V_\perp^\top \Ustar {\Dstar}^{1/2} } \leq \frac{3 \sqrt{B_{\rm in}\norm{\Sigma_\xi}} + 2\sqrt{\tr(\Sigma_\xi)}}{t_0\sqrt{B_{\rm in}}},
\end{align*}
with probability at least $1-e^{-c_2 B_{\rm in}}-e^{-\frac{B_{\rm in}}{8}}$ when $B_{\rm in} \geq c_1 \rstar$. 
\end{proof}

Based on this lemma, we are ready to complete the proof of \Cref{prop:residual}.
\begin{proof}[Proof of \Cref{prop:residual}]
Given any inlier $x = \Ustar {\Dstar}^{1/2} w + \xi$, we have
\begin{align}\label{eq:UB2}
    \norm{V_\perp^\top x} = \norm{V_\perp^\top (\Ustar {\Dstar}^{1/2} w + \xi)}
\leq \norm{V_\perp^\top \Ustar {\Dstar}^{1/2}} \norm{ w } + \norm{\xi}.
\end{align}
By Chebyshev's inequality, we have for any $t>0$
\begin{align*}
\Pr(\norm{w} \geq t) \leq \frac{\EE[\norm{w}^2]}{t^2} = \frac{ \tr\left(\EE\left[w w^\top\right]\right)}{t^2} = \frac{ \tr\left(I_{\rstar}\right)}{t^2} = \frac{\rstar}{t^2}.
\end{align*}
Therefore, $\norm{w} \leq 2\sqrt{\rstar}$ holds with probability at least $\frac{3}{4}$. On the other hand, invoking \Cref{thm:bound_singular_value} with $t=4$ yields:
\begin{align}\label{eq:noise}
    \text{Pr} \left(\norm{ \xi } \geq \sqrt{\tr(\Sigma_\xi)} + 5 \sqrt{\norm{\Sigma_\xi}}\right) \leq e^{- 8}.
\end{align}
Therefore, by combining \Cref{lem:convergence} and \Cref{eq:noise} with  \Cref{eq:UB2}, we have
\begin{align*}
\norm{V_\perp^\top x} \leq \frac{2\sqrt{\rstar}\left(3 \sqrt{B_{\rm in} \norm{\Sigma_\xi}} + 2\sqrt{\tr(\Sigma_\xi)}\right)}{t_0\sqrt{B_{\rm in}}} + \left(\sqrt{\tr(\Sigma_\xi)} + 5 \sqrt{\norm{\Sigma_\xi}}\right),
\end{align*}
with probability at least $\frac{3}{4} - e^{-c_2 B_{\rm in}} - e^{-\frac{B_{\rm in}}{8}} - e^{-8}$. In particular, upon choosing $B / \rstar \geq \max\{ 2c_1, 4/c_2, 32\}$ and invoking \Cref{lem:frac_clean}, we have $B_{\rm in} / \rstar \geq \max\{ c_1, 2/c_2, 16\}$ with probability at least $1-e^{-2(\frac{1}{2}-\epsilon)^2 B}$. It follows that
\begin{align*}
\norm{V_\perp^\top x} \leq \left(\frac{6\sqrt{\rstar}}{t_0} + 5\right) \sqrt{\norm{\Sigma_\xi}} + \left(\frac{1}{t_0}+ 1\right)\sqrt{\tr(\Sigma_\xi)},
\end{align*}
with probability at least $\frac{3}{4} - 2e^{-2\rstar}- e^{-8} - e^{-64(\frac{1}{2}-\epsilon)^2 \rstar}$. Recall that $\cN$ collects the index set of the samples within the set $\cX\backslash \cX_B$. For every $i\in \cN_{\rm in}$ define the Bernoulli random variable:
\begin{align*}
    z_i = \begin{cases}
        1 & \text{if }\norm{V_\perp^\top x_i} \leq \left(\frac{6\sqrt{\rstar}}{t_0} + 5\right) \sqrt{\norm{\Sigma_\xi}} + \left(\frac{1}{t_0}+ 1\right)\sqrt{\tr(\Sigma_\xi)},\\
        0 & \text{otherwise}.
    \end{cases}
\end{align*}
Indeed, given $\rstar\geq 2$ and $\epsilon \leq 0.1$, we have
\begin{align*}
\Pr(z_i=1)&\geq \frac{3}{4} - 2e^{-2\rstar} - e^{-8} - e^{-64\left(\frac{1}{2}-\epsilon \right)^2 \rstar } \geq \frac{3}{4} - 2e^{-4} - 2e^{-8} \geq \frac{7}{10}.  
\end{align*}
Applying \Cref{thm:median} to these Bernoulli random variables with $a = \frac{3}{5}$ leads to: 
\begin{align*}
    \left|\left\{i\in \cN_{\rm in}: \norm{V_\perp^\top x_i} \leq \left(\frac{6\sqrt{\rstar}}{t_0} + 5\right) \sqrt{\norm{\Sigma_\xi}} + \left(\frac{1}{t_0}+ 1\right)\sqrt{\tr(\Sigma_\xi)}\right\}\right|&\geq \frac{3}{5}\left| \cN_{\rm in} \right|\\
    &\geq \frac{3}{5} \left(n - \epsilon n - B \right)\\
    &\geq \frac{1}{2}(n-B)\\
    &= \frac{1}{2}|\cN|,
\end{align*}
with probability at least $1 - e^{-\frac{(1-\epsilon)n-B}{50}} \geq 1 - e^{-\frac{n}{60}}$, where we used $\epsilon \leq 0.1$ and $n \geq 25 B$. Therefore, with the same probability, we have
\begin{align*}
    \mathbf{MedRes}(V) &= \median\left(\left\{\norm{V_\perp^\top x_i}: i\in \cN\right\}\right)\\
    &\leq \left(\frac{6\sqrt{\rstar}}{t_0} + 5\right) \sqrt{\norm{\Sigma_\xi}} + \left(\frac{1}{t_0}+ 1\right)\sqrt{\tr(\Sigma_\xi)}.
\end{align*}
\end{proof}

\subsubsection{Proof of \texorpdfstring{\Cref{thm::first-stage-formal}}{the main result for first stage algorithm}}
Equipped with \Cref{prop::phase1}, \Cref{prop::phase1.5}, \Cref{prop:residual}, we are now ready to present the proof of \Cref{thm::first-stage-formal}.

\begin{proof}[Proof of \Cref{thm::first-stage-formal}]
For each $x \sim \cP$, we have 
\begin{align}\label{eq::bound}
\begin{aligned}
    \norm{(I_d - {\rm P}_{\cV}) x} &= \norm{V_\perp^\top x}  \\
    &= \norm{V_\perp^\top \Ustar {\Ustar}^\top x}\\
&= \norm{V_\perp^\top \Ustar {\Dstar}^{1/2}{\Dstar}^{-1/2}{\Ustar}^\top x} \\
&\leq \norm{V_\perp^\top \Ustar {\Dstar}^{1/2}} \norm{{\Dstar}^{-1/2}{\Ustar}^\top x}\\
&\leq \frac{\norm{x}}{\sqrt{\gammastar_{\min}}} \norm{V_\perp^\top \Ustar {\Dstar}^{1/2}}.
\end{aligned}
\end{align}
A simple union bound implies that \Cref{prop::phase1}, \Cref{prop::phase1.5}, \Cref{prop:residual} are all satisfied with probability at least $1-e^{-c_4n}-e^{-\frac{n}{60}}-e^{-\frac{n}{250}}$. Upon choosing $n \geq \max\{\frac{1}{c_4}, 250\} \log(\frac{3}{\delta})$, we have $1-e^{-c_4n}-e^{-\frac{n}{60}}-e^{-\frac{n}{250}} \geq 1-\delta$. This implies that \Cref{alg:first_stage} terminates with a batch size $\rstar\leq B\leq \max\{ 2c_1,\, 4/c_2,\, 32\}\cdot \rstar$ and, upon termination, satisfies 
\begin{align}\label{eq::bound2}
\norm{V_\perp^\top \Ustar {\Dstar}^{1/2}} \leq \frac{4C}{t_0^2}\left(5\sqrt{\rstar\norm{\Sigma_\xi}} + \sqrt{\tr(\Sigma_\xi)}\right).
\end{align}
Combining \Cref{eq::bound2} with \Cref{eq::bound} leads to the desired bound.

Finally, we analyze the runtime of \Cref{alg:first_stage}. The algorithm terminates with high probability when $B = \cO(\rstar)$, and since $B$ doubles in each iteration, the number of iterations required to reach the stopping criterion is at most $\cO(\log (\rstar))$ with probability at least $1-\delta$. In each iteration, computing the left singular vectors $V$ of the subsampled matrix $X_B \in \mathbb{R}^{d \times B}$ takes $\cO(d B^2) = \cO(d {\rstar}^2)$ flops, while evaluating $\mathbf{MedRes}(V)$ requires $\cO(nd\rstar)$ flops. Therefore, the total runtime is 
$\cO \left( \log (\rstar) \cdot (d{\rstar}^2 + nd\rstar) \right) = {\cO}(nd\rstar\log (\rstar))$.
\end{proof}

\subsection{Analysis of the Second-stage Algorithm}
\label{sec:second_stage} 

In this section, we present the theoretical analysis of the second-stage algorithm (\Cref{alg:second_stage}). In particular, our goal is to prove the formal version of \Cref{thm:alg_2}, which we present below.

\begin{sloppypar}
    \begin{theorem}\label{thm::second-stage-formal}
Suppose that the conditions of \Cref{thm::first-stage-formal} hold. Additionally, 
    fix $\epsilon \leq \min\left\{0.1,\,0.9\cdot\left(1- e^{-\frac{c_2}{4}}\right)\right\}$, $C'\geq \max\{2c_1,8/c_2\}$ with $c_1,c_2$ appearing in \Cref{thm:lgd}, and $n\geq 11C'\max\left\{\hat r, \log\left(\frac{3}{\delta} \log\left(\frac{1}{\delta}\right)\right)\right\}$. Then, \Cref{alg:second_stage} terminates in $\cO({\rstar}^3 e^{\rstar})$ time and, with probability at least $1-3\delta$, satisfies:
\begin{align*}
\tilde r = \rstar, \quad \text{and}\quad 
\norm{ (V\widehat U_{k})(V\widehat U_{k})^\top - \Ustar {\Ustar}^\top } \leq \frac{4(t_0 + C)C}{t_0^3} \left( 6\sqrt{\frac{ \rstar \norm{\Sigma_\xi}}{\gammastar_{\min}}}  + \sqrt{\frac{\tr(\Sigma_\xi)}{\gammastar_{\min}}}\right).
\end{align*}
\end{theorem}
\end{sloppypar}

Similar to our analysis of the first-stage algorithm, we begin by outlining the high-level idea of our proof strategy. Consider the following two events:
\begin{align*}
    \cE_1 &= \left\{\text{For some $j\in [T]$, $\widehat X_j$ only contains inliers.}\right\},\\
    \cE_2 &= \left\{\text{For all $j\in [T]$, at least half of the samples in $\widehat X_j$ are inliers.}\right\}.
\end{align*}
Conditioned on $\cE_1$, let $j^\star$ be the index for which the batch $\widehat X_{j^\star}$ only contains inliers. Then, we will show that
\begin{align*}
\hat{\gamma}_{\rstar+1} := \min_{j\in [T]}\left(\widehat \Sigma_{\rstar+1}^{(j)}\right)^2 \leq \left(\widehat \Sigma_{\rstar+1}^{(j^\star)}\right)^2 = \cO(\norm{\Sigma_\xi}).
\end{align*}
Moreover, conditioned on $\cE_2$, we will establish that
\begin{align*}
\hat{\gamma}_{\rstar} := \min_{j\in [T]}\left(\widehat \Sigma_{\rstar}^{(j)}\right)^2 = \Omega(\gammastar_{\min})\gg \cO(\norm{\Sigma_\xi}).
\end{align*}
These two relations together establish a significant gap between $\hat{\gamma}_{\rstar+1}$ and $\hat{\gamma}_{\rstar}$.  
In particular, the first relation implies that the computed index \( \tilde{r} \) in Line~\ref{line::success-second-stage} satisfies \( \tilde{r} \leq \rstar \), while the second ensures that \( \tilde{r} \geq \rstar \). Together, they establish that \( \tilde{r} = \rstar \), thereby guaranteeing the exact recovery of the true subspace dimension.  
Finally, we show that by selecting the batch size \( B \) and the number of iterations \( T \) as specified in Lines~\ref{line::B-second-stage} and~\ref{line::T-second-stage} of \Cref{alg:second_stage}, the probability of the desired events occurring satisfies
$\Pr(\cE_1 \cap \cE_2) \geq 1 - 2\delta$.

To formalize the above argument, we begin with the following lemma, which establishes a connection between the true covariance matrix \( \Sigmastar \) and its projected counterpart \( V^\top \Sigmastar V \).

\begin{lemma}
\label{lem:min_max}
Given $\Sigmastar \in \RR^{d\times d}$ and $\widehat{\Sigma}^{\star} = V^\top \Sigmastar V \in \RR^{\hat r\times \hat r}$, we have $\gamma_k(\widehat{\Sigma}^{\star}) \leq \gamma_k(\Sigmastar)$ and $\sqrt{\gamma_k(\widehat{\Sigma}^{\star})} \geq \sqrt{\gamma_k(\Sigmastar)} - \norm{V_\perp^\top \Ustar {\Dstar}^{1/2}}$, for every $k\in [\hat r]$.
\end{lemma}
\begin{proof}
By the min-max theorem for eigenvalues \cite[Theorem 3.1.2]{horn1994topics}, we have
\begin{align*}
\begin{aligned}
\gamma_k(\widehat{\Sigma}^{\star}) &= \sup_{\substack{\cM\subset \RR^{\hat r} \\ \dim(\cM) = k}} \inf_{x\in \cM} \frac{x^\top \widehat{\Sigma}^{\star} x}{x^\top x}\\
&= \sup_{\substack{M\in \RR^{\hat r \times k}\\ \text{rank}(M) = k}} \inf_{y \in \RR^{k}} \frac{(My)^\top V^\top \Sigmastar V (My)}{(My)^\top (My)}\\
&= \sup_{\substack{M\in \RR^{\hat r \times k}\\ \text{rank}(M) = k}} \inf_{y \in \RR^{k}} \frac{(V My)^\top \Sigmastar (V My)}{(V My)^\top (V My)}\\
&= \sup_{\substack{M\in \RR^{\hat r \times k}\\ \text{rank}(M) = k}} \inf_{z \in \textbf{col}(VM)} \frac{z^\top \Sigmastar z}{z^\top z}\\
&\stackrel{(a)}{=} \sup_{\substack{W \in \RR^{d \times k}\\ \text{rank}(V^\top W) = k}} \inf_{z \in \textbf{col}(W)} \frac{z^\top \Sigmastar z}{z^\top z}\\
&\stackrel{(b)}{\leq} \sup_{\substack{W \in \RR^{d \times k}\\ \text{rank}(W) = k}} \inf_{z \in \textbf{col}(W)} \frac{z^\top \Sigmastar z}{z^\top z} = \gamma_k(\Sigmastar),
\end{aligned}
\end{align*}
for every $k\in [\hat r]$, where $(a)$ follows from a change of variable $W = VM$, and $(b)$ holds since $k = \text{rank}(V^\top W) \leq \text{rank}(W) \leq k$, which leads to $\text{rank}(W) = k$. On the other hand,
\begin{align*}
\sqrt{\gamma_k(\widehat{\Sigma}^{\star})} &= \sqrt{\gamma_k\left(V^\top \Ustar \Dstar {\Ustar}^\top V\right)}\\
&= \sigma_{k}(V^\top \Ustar {\Dstar}^{1/2})\\
&= \sigma_{k}(V V^\top \Ustar {\Dstar}^{1/2})\\
&= \sigma_{k}((I_{d} - V_\perp V_\perp^\top) \Ustar {\Dstar}^{1/2})\\
&\geq \sigma_{k}(\Ustar {\Dstar}^{1/2}) - \sigma_{\max}(V_\perp V_\perp^\top \Ustar {\Dstar}^{1/2})\\
&= \sqrt{\gamma_k(\Sigmastar)} - \norm{V_\perp^\top \Ustar {\Dstar}^{1/2}},
\end{align*}
for every $k\in [\hat r]$. This completes the proof.
\end{proof}

Returning to \Cref{alg:second_stage}, recall that each batch \( \widehat{\cX}_j \) is selected independently and uniformly at random from \( \widehat{\cX} \) and is stored in the matrix \( \widehat{X}_j \in \RR^{\hat{r} \times B} \). Let the empirical covariance matrix of the \( j^{\text{th}} \) batch be denoted by  
$\widehat{\Sigma}_j = \frac{1}{B} \widehat{X}_j \widehat{X}_j^\top.$
Furthermore, let \( B_j^{\rm in} \) and \( B_j^{\rm out} \) represent the number of inliers and outliers in \( \widehat{\cX}_j \), respectively. Denote the ensemble of the inliers and outliers in \( \widehat{\cX}_j \) by $\widehat{X}_j^{\rm in} \in \RR^{\hat{r} \times B_j^{\rm in}}$ and $\widehat{X}_j^{\rm out} \in \RR^{\hat{r} \times B_j^{\rm out}}$, respectively. Finally, define the empirical covariance matrices of the inliers and outliers in \( \widehat{\cX}_j \) as $\widehat{\Sigma}^{\rm in}_j = \frac{1}{B_j^{\rm in}} \widehat{X}_j^{\rm in} (\widehat{X}_j^{\rm in})^\top$ and $\widehat{\Sigma}^{\rm out}_j = \frac{1}{B_j^{\rm out}} \widehat{X}_j^{\rm out} (\widehat{X}_j^{\rm out})^\top$.
It follows that $\widehat{\Sigma}_j = (1-\mu_j) \widehat{\Sigma}^{\rm in}_j + \mu_j \widehat{\Sigma}^{\rm out}_j$, where $\mu_j := {B_j^{\rm out}}/{B}$ is the fraction of outliers in $\widehat{\cX}_j$. 

Let the eigen-decomposition of the projected covariance matrix \( \widehat{\Sigma}^\star = V^\top \Sigmastar V \) be given by $\widehat{\Sigma}^{\star} = \widehat{U}^{\star} \widehat{D}^{\star} (\widehat{U}^{\star})^\top$,
where \( \widehat{U}^{\star} \in O(\hat{r}, \rstar) \) and \( \widehat{D}^{\star} \in \RR^{\rstar \times \rstar} \) is a diagonal matrix satisfying \( \widehat{D}^{\star} \preceq \Dstar \). Moreover, let \( \widehat{U}^{\star}_{\perp} \in O(\hat{r}, \hat{r} - \rstar) \) denote the orthogonal complement of \( \widehat{U}^{\star} \). Our next lemma establishes concentration bounds on \( \widehat{\Sigma}^{\rm in}_j \) in terms of $\widehat{U}^{\star}$ and $\widehat{U}^{\star}_{\perp}$.

\begin{lemma}
\label{lem:eigen_gap}
For any fixed $j\in [T]$, with probability at least $1 - e^{-B_j^{\rm in} / 2}$, we have
\begin{align*}
\norm{(\widehat{U}^{\star}_\perp)^\top \widehat{\Sigma}^{{\rm in}}_j \widehat{U}^{\star}_{\perp}} \leq \left( 2 + \sqrt{\frac{\hat r}{B_j^{\rm in}}} \right)^2 \norm{\Sigma_\xi}.
\end{align*}
Moreover, assuming $B_j^{\rm in} \geq c_1 \rstar$, with probability at least $1-e^{-c_2 B_j^{\rm in}}-e^{-B_j^{\rm in}/2}$, we have
\begin{align*}
\gamma_{\min} \left((\widehat{U}^{\star})^\top \widehat{\Sigma}^{{\rm in}}_j \widehat{U}^{\star} \right)
\geq \left( \frac{t_0}{2}\left(\sqrt{\gammastar_{\min}} - \norm{V_\perp^\top \Ustar {\Dstar}^{1/2}}\right) - \left( 2 + \sqrt{\frac{\hat r}{B_j^{\rm in}}} \right) \sqrt{\norm{\Sigma_\xi}} \right)^2.
\end{align*}
\end{lemma}
\begin{proof}
Given the ensemble of inliers $\widehat{X}_j^{\rm in}$, we decompose it as $\widehat{X}_j^{\rm in} = \widetilde{X}_j^{\rm in} + \widehat{\Xi}_j^{\rm in}$,
where $\widetilde{X}_j^{\rm in}$ denotes the ensemble of projected clean components, and $\widehat{\Xi}_j^{\rm in}$ corresponds to the ensemble of projected Gaussian noise associated with each inlier in $\widehat{\cX}_j$. Each column of $\widetilde{X}_j^{\rm in}$ is drawn from a distribution with a zero mean and covariance $\widehat{\Sigma}^\star = V^\top \Sigmastar V$. Moreover, each column of $\widehat{\Xi}_j^{\rm in}$ is drawn from a Gaussian distribution with a zero mean and covariance $V^\top \Sigma_\xi V$.
 Since $(\widehat{U}^{\star}_{\perp})^\top \widetilde{X}_j^{\rm in} = 0$, we have
\begin{align*}
\begin{aligned}
\norm{ (\widehat{U}^{\star}_{\perp})^\top \widehat{\Sigma}^{\rm in}_j \widehat{U}^{\star}_{\perp} } &= \frac{\norm{\left((\widehat{U}^{\star}_{\perp})^\top \widehat X_j^{\rm in}\right)\left((\widehat{U}^{\star}_{\perp})^\top \widehat X_j^{\rm in}\right)^\top }}{B_j^{\rm in}}\\
&= \frac{\norm{ \left((\widehat{U}^{\star}_{\perp})^\top \widehat{\Xi}_j^{\rm in}\right)\left((\widehat{U}^{\star}_{\perp})^\top \widehat{\Xi}_j^{\rm in}\right)^\top }}{B_j^{\rm in}}\\
&\leq \frac{\norm{ (\widehat{\Xi}_j^{\rm in}) (\widehat{\Xi}_j^{\rm in})^\top }}{B_j^{\rm in}} = \left( \frac{\sigma_{\max}(\widehat{\Xi}_j^{\rm in})}{\sqrt{B_j^{\rm in}}} \right)^2.
\end{aligned}
\end{align*}
By \Cref{thm:bound_singular_value}, with probability at least $1 - e^{-B_j^{\rm in} / 2}$, we have
\begin{align}\label{eq::UB_noise_max}
\frac{\sigma_{\max}(\widehat{\Xi}_j^{\rm in})}{\sqrt{B_j^{\rm in}}}
\leq 2 \sqrt{\norm{\Sigma_\xi}} + \sqrt{\frac{\tr(V^\top \Sigma_\xi V)}{B_j^{\rm in}}}
\leq \left( 2 + \sqrt{\frac{\hat r}{B_j^{\rm in}}} \right) \sqrt{\norm{\Sigma_\xi}}.
\end{align}
This completes the proof of the first statement. To prove the second statement, recall that, given any random vector $x\sim \cP$, its projection $\widehat x = V^\top x$ has zero mean and covariance $\widehat{\Sigma}^{\star} = \widehat{U}^{\star} \widehat{D}^{\star} (\widehat{U}^{\star})^\top$. Therefore, $\widehat w = (\widehat{D}^{\star})^{-\frac{1}{2}}(\widehat{U}^{\star})^\top \widehat x$ is normalized, and can be easily verified to satisfy \Cref{asp:anti_concentration}. This implies that each column of $(\widehat{D}^{\star})^{-\frac{1}{2}}(\widehat{U}^{\star})^\top \widetilde{X}_j^{\rm in}$ independently satisfies \Cref{asp:anti_concentration}, and hence, by \Cref{thm:lgd}, we have
\begin{align}\label{eq::signal_max}
\frac{\sigma_{\min}\left((\widehat{D}^{\star})^{-\frac{1}{2}}(\widehat{U}^{\star})^\top \widetilde{X}_j^{\rm in}\right)}{\sqrt{B_j^{\rm in}}} \geq \sqrt{\frac{3t_0^2}{8}}\geq \frac{t_0}{2},
\end{align}
with probability at least $1-e^{-c_2 B_j^{\rm in}}$, provided that $B_j^{\rm in} \geq c_1 \rstar$. It follows that
\begin{align*}
\gamma_{\min} \left((\widehat{U}^{\star})^\top \widehat{\Sigma}^{\rm in}_j \widehat{U}^{\star} \right) &= \frac{\gamma_{\min} \left( ((\widehat{U}^{\star})^\top \widehat X_j^{\rm in})((\widehat{U}^{\star})^\top \widehat X_j^{\rm in})^\top \right)}{B_j^{\rm in}}\\
&= \left(\frac{\sigma_{\min}\left((\widehat{U}^{\star})^\top \widehat X_j^{\rm in}\right)}{\sqrt{B_j^{\rm in}}}\right)^2\\
&\geq \left(\frac{\sigma_{\min}\left((\widehat{U}^{\star})^\top \tilde X_j^{\rm in}\right)}{\sqrt{B_j^{\rm in}}} - \frac{\sigma_{\max}\left((\widehat{U}^{\star})^\top V^\top \widehat{\Xi}_j^{\rm in}\right)}{\sqrt{B_j^{\rm in}}}\right)^2\\
&\geq \left(\frac{\gamma_{\min}\left(({\widehat{D}^{\star}})^{1/2}\right)\sigma_{\min}\left((\widehat{D}^{\star})^{-\frac{1}{2}}(\widehat{U}^{\star})^\top \tilde X_j^{\rm in}\right)}{\sqrt{B_j^{\rm in}}} - \frac{\sigma_{\max}(\widehat{\Xi}_j^{\rm in})}{\sqrt{B_j^{\rm in}}}\right)^2\\
&\geq \left( \frac{t_0}{2}\left(\sqrt{\gammastar_{\min}} - \norm{V_\perp^\top \Ustar {\Dstar}^{1/2}}\right) - \left( 2 + \sqrt{\frac{\hat r}{B_j^{\rm in}}} \right) \sqrt{\norm{\Sigma_\xi}} \right)^2,
\end{align*}
with probability at least $1-e^{-c_2 B_j^{\rm in}}-e^{-B_j^{\rm in}/2}$. In the last inequality, we used \Cref{eq::UB_noise_max}, \Cref{eq::signal_max}, and \Cref{lem:min_max}, which leads to $\gamma_{\min}\left(({\widehat{D}^{\star}})^{1/2}\right)\geq \sqrt{\gammastar_{\min}} - \norm{V_\perp^\top \Ustar {\Dstar}^{1/2}}$. 
\end{proof}

Based on \Cref{lem:eigen_gap}, we now derive an upper bound on $\hat{\gamma}_{\rstar+1}$ and a lower bound on $\hat{\gamma}_{\rstar}$. As previously discussed, these bounds play a key role in the proof of \Cref{thm::second-stage-formal}.

\begin{sloppypar}
    \begin{lemma}
\label{prop:eigen_gap}
Recall that $\hat{\gamma}_i = \min_{j\in [T]} \left(\hat{\sigma}_i^{(j)}\right)^2$, where $\hat{\sigma}_i^{(j)}$ is the $i^{\text{th}}$ largest singular value of the $j^{\text{th}}$ subsample matrix $\widehat X_j$. Assuming that $B\geq 2c_1\hat r$, then with probability at least $\left(1\!-\!\left(1\!-\! \frac{\binom{(1-\epsilon) n}{B}}{\binom{n}{B}} \right)^T \right)\left( 1 - T \left(e^{-2(\frac{1}{2} - \epsilon)^2 B} \!+\! e^{-\frac{1}{4}B} \!+\! e^{-\frac{c_2}{2}B} \right) \right)$, we have
\begin{align*}
\hat{\gamma}_{\rstar+1} \leq  9 \norm{\Sigma_\xi}, 
\quad \text{and} \quad 
\hat{\gamma}_{\rstar} \geq \frac{1}{2} \left( \frac{t_0}{2}\left(\sqrt{\gammastar_{\min}} - \norm{V_\perp^\top \Ustar {\Dstar}^{1/2}}\right) - 3 \sqrt{\norm{\Sigma_\xi}} \right)^2.
\end{align*}
\end{lemma}
\end{sloppypar}

\begin{proof}
Define the following events:
\begin{align*}
\cE_1 &= \left\{\text{For some $j\in [T]$, $\widehat X_j$ only contains inliers.}\right\},\\
\cE_2^{(j)} &= \left\{ \text{At least half of the samples in $\widehat X_j$ are inliers.} \right\}, \text{ for } j\in [T],\\
\cE_3^{(j)} &= \left\{ \text{The lower and upper bounds of \Cref{lem:eigen_gap} hold for batch $j$.} \right\}, \text{ for } j\in [T].
\end{align*}
Further, define the events $\cE_2 = \bigcap_{j=1}^{T} \cE_2^{(j)}$ and $\cE_3 = \bigcap_{j=1}^{T} \cE_3^{(j)}$. First, we have
\begin{align*}
\Pr(\cE_1) &=1-  \prod_{j=1}^{T} \Pr\left\{\text{Batch $j$ contains at least one outlier.}\right\} =1-\left(1- \frac{\binom{(1-\epsilon) n}{B}}{\binom{n}{B}} \right)^T.
\end{align*}
Moreover, by \Cref{lem:frac_clean}, we have $\Pr\left({\cE_2^{(j)}}^c \right) \leq e^{-2(\frac{1}{2} - \epsilon)^2 B}$ for any $j\in[T]$.
Conditioned on $\cE_2^{(j)}$ and recalling $B \geq 2c_1\hat r\geq  2c_1\rstar$, we have $B_j^{\rm in} \geq  c_1 \rstar$. It follows from \Cref{lem:eigen_gap} that 
\begin{align*}
\Pr\left({\cE_3^{(j)}}^c \ \middle|\ \cE_2^{(j)} \right) \leq e^{-\frac{1}{2}B_j^{\rm in}} + e^{-c_2B_j^{\rm in}} \leq e^{-\frac{1}{4}B} + e^{-\frac{c_2}{2}B}.
\end{align*}
Utilizing these bounds, we have
\begin{align*}
\Pr(\cE_2\cap \cE_3) &= \Pr(\cE_2)\Pr(\cE_3 \mid \cE_2)
= \Pr\left(\bigcap_{j=1}^{T} \cE_2^{(j)}\right)\Pr\left(\bigcap_{j=1}^{T}\cE_3^{(j)} \mid \cE_2\right)\\
&= \left(1 - \Pr\left(\bigcup_{j=1}^{T} {\cE_2^{(j)}}^c\right)\right) \left(1- \Pr\left(\bigcup_{j=1}^{T}{\cE_3^{(j)}}^c \mid \cE_2\right)\right)\\
&\geq \left(1 - \sum_{j=1}^T \Pr\left( {\cE_2^{(j)}}^c\right)\right) \left(1- \sum_{j=1}^T \Pr\left({\cE_3^{(j)}}^c \mid \cE_2\right)\right)\\
&\geq \left(1 - T e^{-2(\frac{1}{2} - \epsilon)^2 B}\right)\left(1 - T \left(e^{-\frac{1}{4}B} + e^{-\frac{c_2}{2}B}\right) \right)\\
&\geq 1 - T\left(e^{-2(\frac{1}{2} - \epsilon)^2 B} + e^{-\frac{1}{4}B} + e^{-\frac{c_2}{2}B}\right).
\end{align*}
It follows that the probability of the event $\cE_1 \cap \cE_2 \cap \cE_3$ occurring can be bounded as
\begin{align*}
\Pr(\cE_1\cap\cE_2\cap\cE_3) 
&\geq \left(1-\left(1- \frac{\binom{(1-\epsilon) n}{B}}{\binom{n}{B}} \right)^T \right)\left( 1 - T\left(e^{-2(\frac{1}{2} - \epsilon)^2 B} + e^{-\frac{1}{4}B} + e^{-\frac{c_2}{2}B}\right) \right).
\end{align*}
Conditioned on the event $\cE_1 \cap \cE_2 \cap \cE_3$, we now establish the desired bounds for $\hat{\gamma}_{\rstar+1}$ and $\hat{\gamma}_{\rstar}$. Again, by the min-max theorem for eigenvalues and \Cref{lem:eigen_gap}, we have, for every $j\in [T]$,
\begin{align*}
\gamma_{\rstar+1}(\widehat{\Sigma}^{\rm in}_j)
&= \inf_{\substack{\cM\subset \RR^{B} \\ \dim(\cM) = B-\rstar}}
\sup_{x\in \cM} \frac{x^\top \widehat{\Sigma}^{\rm in}_j x}{x^\top x}\\
&= \inf_{\substack{M\in \RR^{B \times (B-\rstar)}\\ \text{rank}(M) = B-\rstar}} \sup_{x = My} \frac{x^\top \widehat{\Sigma}^{\rm in}_j x}{x^\top x}\\
&\stackrel{(a)}{\leq} \sup_{x = \widehat{U}^{\star}_{\perp} y} \frac{x^\top \widehat{\Sigma}^{\rm in}_j x}{x^\top x}
= \norm{ (\widehat{U}^{\star}_{\perp})^\top \widehat{\Sigma}^{\rm in}_j (\widehat{U}^{\star}_{\perp}) } \\
&\stackrel{(b)}{\leq} \left(2+ \sqrt{\frac{\hat r}{B_j^{\rm in}}}\right)^2 \norm{\Sigma_\xi} \\
&\stackrel{(c)}{\leq} 9 \norm{\Sigma_\xi}.
\end{align*}
Here, $(a)$ is implied by setting $M=\widehat{U}^{\star}_{\perp}$, $(b)$ follows from \Cref{lem:eigen_gap}, and $(c)$ is due to $\hat r \leq \frac{B}{2c_1} \leq \frac{B}{2} \leq  B_j^{\rm in}$. Under the event $\cE_1 \cap \cE_2 \cap \cE_3$, the above inequality holds uniformly for all $j\in [T]$. Therefore, we have
\begin{align}
\label{eq:gamma_rstar+1}
\hat{\gamma}_{\rstar+1} := \min_{j\in [T]} \gamma_{\rstar+1}(\widehat \Sigma_j) \leq  9 \norm{\Sigma_\xi}.
\end{align}
Next, we establish the desired lower bound on $\hat{\gamma}_{\rstar}$.
Again, by the min-max theorem for eigenvalues and \Cref{lem:eigen_gap}, we have, for every $j\in [T]$,
\begin{align*}
\gamma_{\rstar}(\widehat{\Sigma}_j^{\rm in}) 
\geq \gamma_{\min} ((\widehat{U}^{\star})^\top \widehat{\Sigma}^{\rm in}_j \widehat{U}^{\star})
\geq \left( \frac{t_0}{2}\left(\sqrt{\gammastar_{\min}} - \norm{V_\perp^\top \Ustar {\Dstar}^{1/2}}\right) - 3 \sqrt{\norm{\Sigma_\xi}} \right)^2,
\end{align*}
 where again we used the fact that $\hat r \leq  B_j^{\rm in}$. Under the event $\cE_1 \cap \cE_2 \cap \cE_3$, the above inequality holds uniformly for all $j\in [T]$. Therefore, we obtain
\begin{align}
\label{eq:gamma_rstar}
\hat{\gamma}_{\rstar} := \min_{j\in [T]} \gamma_{\rstar}(\widehat \Sigma_j) \stackrel{(a)}{\geq} \min_{j\in [T]} \frac{1}{2}\gamma_{\rstar}(\widehat \Sigma_j^{\rm in}) 
\geq \frac{\left( \frac{t_0}{2}\left(\sqrt{\gammastar_{\min}} - \norm{V_\perp^\top \Ustar {\Dstar}^{1/2}}\right) - 3 \sqrt{\norm{\Sigma_\xi}} \right)^2}{2},
\end{align}
where in $(a)$, we used the fact that $\widehat{\Sigma}_j = (1-\mu_j) \widehat{\Sigma}^{\rm in}_j + \mu_j \widehat{\Sigma}^{\rm out}_j$ with $\mu_j := {B_j^{\rm out}}/{B}\leq 1/2$, which in turn implies $\gamma_k(\widehat{\Sigma}_j) \geq (1/2)\cdot \gamma_k(\widehat{\Sigma}^{\rm in}_j)$ for $k\in[\hat{r}]$. 
\end{proof}

As a final technical lemma before presenting the proof of \Cref{thm::second-stage-formal}, we establish that, under appropriate choices of $T$ and $B$, the probability of the event $\cE_1 \cap \cE_2 \cap \cE_3$ occuring, as characterized in \Cref{prop:eigen_gap}, can be lower bounded by $1-2\delta$.

\begin{sloppypar}
    \begin{lemma}
\label{lem:failure_prob}
The following statements hold:
\begin{itemize}
    \item If $n\geq 11B$ and $T \geq \left( \frac{1}{1-1.1\epsilon} \right)^B \log (\frac{1}{\delta})$, then
\begin{align}
\label{eq:T_lower}
1-\left(1- \frac{\binom{(1-\epsilon) n}{B}}{\binom{n}{B}} \right)^T  \geq 1-\delta.
\end{align}
\item If $\epsilon \leq 0.1$ and $T \leq \frac{1}{3}e^{c_2 B/2} \delta$, then
\begin{align}
\label{eq:T_upper}
1 - T\left(e^{-2(\frac{1}{2} - \epsilon)^2 B} + e^{-\frac{1}{4}B} + e^{-\frac{c_2}{2}B}\right) \geq 1-\delta.
\end{align}
\item If $\epsilon \leq \min\left\{0.1,\,0.9\cdot\left(1- e^{-\frac{c_2}{4}}\right)\right\}$, $n\geq 11B$, $B \geq \frac{4}{c_2} \log\left(\frac{3}{\delta} \log\left(\frac{1}{\delta}\right)\right)$, and $T = \left( \frac{1}{1-1.1\epsilon} \right)^B \log (\frac{1}{\delta})$, then
\begin{align*}
     \left(1-\left(1- \frac{\binom{(1-\epsilon) n}{B}}{\binom{n}{B}} \right)^T \right)\left(1 - T\left(e^{-2(\frac{1}{2} - \epsilon)^2 B} + e^{-\frac{1}{4}B} + e^{-\frac{c_2}{2}B}\right) \right)\geq 1-2\delta.
\end{align*}
\end{itemize}
\end{lemma}
\end{sloppypar}

\begin{proof}
To prove the first statement, let $\alpha := \frac{\binom{(1-\epsilon) n}{B}}{\binom{n}{B}}$. When $n\geq 11B$, we have
\begin{align*}
\alpha = \frac{\binom{(1-\epsilon) n}{B}}{\binom{n}{B}}=\frac{((1-\epsilon) n) \cdots ((1-\epsilon) n - B+1)}{n (n-1) \cdots (n-B+1)} \geq \left( \frac{(1-\epsilon) n - B+1}{n-B+1} \right)^B \geq (1-1.1\epsilon)^B.
\end{align*}
Therefore, the condition $T \geq \left( \frac{1}{1-1.1\epsilon} \right)^B \log (\frac{1}{\delta})$ implies that $T\geq \frac{1}{\alpha}\log (\frac{1}{\delta})$. On the other hand, $1 - \frac{1}{x} \leq \log (x)$ for any $x > 0$. Therefore, 
\begin{align*}
    T\geq \frac{1}{\alpha}\log \left(\frac{1}{\delta}\right)\geq \frac{\log (\frac{1}{\delta})}{\log (\frac{1}{1-\alpha})}\implies 1-(1-\alpha)^T\geq 1-\delta.
\end{align*}
To prove the second statement, let $\beta := e^{-2(\frac{1}{2} - \epsilon)^2 B} + e^{-\frac{B}{4}} + e^{-\frac{c_2 B}{2}}$. When $\epsilon\leq 0.1$, we have $\beta  \leq 2  e^{-\frac{B}{4}} + e^{-\frac{c_2 B}{2}}\leq 3 e^{-\frac{c_2 B}{2}}$,
where we used the fact that $c_2\leq 1/2$ due to \Cref{thm:lgd}. Therefore, the condition $T \leq \frac{1}{3}e^{c_2 B/2} \delta$ implies $T\leq \frac{\delta}{\beta}$, which in turn yields $1-T\cdot \beta\geq 1-\delta$.

Finally, to establish the third statement, it suffices to show that, under the given conditions, the first and second statements hold simultaneously.
In particular, we show that if $\epsilon \leq 0.9 \cdot \left(1 - e^{-\frac{c_2}{4}}\right)$ and $B \geq \frac{4}{c_2} \log\left(\frac{3}{\delta} \log\left(\frac{1}{\delta}\right)\right)$, then the required upper bound on $T$ in the second statement exceeds its required lower bound in the first statement. This implies that $T = \left( \frac{1}{1 - 1.1\epsilon} \right)^B \log\left(\frac{1}{\delta}\right)$ is a feasible choice for both statements. From $\epsilon \leq 0.9 \cdot \left(1 - e^{-\frac{c_2}{4}}\right)$, we obtain $\frac{c_2}{4}\geq \log\left(\frac{1}{1-1.1\epsilon}\right)$. Moreover, from $B \geq \frac{4}{c_2} \log\left(\frac{3}{\delta} \log\left(\frac{1}{\delta}\right)\right)$, we obtain $\frac{c_2}{4}B \geq \log\left(\frac{3}{\delta} \log\left(\frac{1}{\delta}\right)\right)$. Combining these two inequalities leads to 
\begin{align*}
\left( \frac{c_2}{2} - \log \left( \frac{1}{1-1.1\epsilon} \right) \right) B \geq \log\left(\frac{3}{\delta} \log\left(\frac{1}{\delta}\right)\right)\implies \left( \frac{1}{1-1.1\epsilon} \right)^B \log (\frac{1}{\delta})\leq \frac{1}{3}e^{c_2 B/2} \delta.
\end{align*}
\end{proof}

Equipped with the above lemmas, we are ready to present the proof of \Cref{thm::second-stage-formal}.

\subsubsection{Proof of \texorpdfstring{\Cref{thm::second-stage-formal}}{the main result}}
Before proving the final result, we present the following useful lemma for bounding the $\ell_2$-error between different projection operators.
\begin{lemma}[Lemma 2.5 in \cite{chen2020spectral}]
\label{lem:dk}
Suppose that $U, V \in O(d,r)$, and $U_\perp, V_\perp \in O(d,d-r)$ are the orthogonal complements of $U, V$ respectively. Then, we have
\begin{align*}
\norm{ U U^\top - V V^\top } = \norm{ V^\top U_\perp } = \norm{U^\top V_\perp}.
\end{align*}
\end{lemma}

\begin{proof}[Proof of \Cref{thm::second-stage-formal}]
First, we establish that $\tilde r = \rstar$. Since the conditions of \Cref{thm::first-stage-formal} are assumed, we have $5\sqrt{\rstar \norm{\Sigma_\xi} + \tr(\Sigma_\xi)} \leq \frac{t_0^2}{4C+4t_0} \sqrt{\gammastar_{\min}}$. Moreover, under the same conditions, \Cref{eq::bound2} implies that $\norm{V_\perp^\top \Ustar {\Dstar}^{1/2}} \leq \frac{4C}{t_0^2}\left(5\sqrt{\rstar\norm{\Sigma_\xi}} + \sqrt{\tr(\Sigma_\xi)}\right)$ with probability at least $1-\delta$. Therefore, with probability at least $1-\delta$, we have
\begin{align*}
9 \norm{\Sigma_\xi} <   \frac{1}{2} \left( \frac{t_0}{2}\left(\sqrt{\gammastar_{\min}} - \norm{V_\perp^\top \Ustar {\Dstar}^{1/2}}\right) - 3 \sqrt{\norm{\Sigma_\xi}} \right)^2.
\end{align*}
Combining the above inequality with \Cref{prop:eigen_gap} and \Cref{lem:failure_prob}, we have $\hat{\gamma}_{\rstar+1} \leq 9 \norm{\Sigma_\xi} < \hat{\gamma}_{\rstar}$, with probability at least $1 - 3\delta$, which in turn shows that \Cref{alg:second_stage} assigns $\tilde{r} = \rstar$.

Next, we present the proof of the error guarantee. According to Line~\ref{line::min-assignment}, we have $k = \argmin_{j\in [T]} \gamma_{\rstar+1}(\widehat{\Sigma}_j)$. Denote the top-$\rstar$ eigen-decomposition of $\widehat \Sigma_{k}$ by $\widehat U_{k} \widehat D_{k} \widehat U_{k}^\top$, where $\widehat D_{k}\in\mathbb{R}^{\rstar\times \rstar}$ is a diagonal matrix collecting the top-$\rstar$ eigenvalues of $\widehat \Sigma_{k}$. It follows from the definition $\hat{\gamma}_{\rstar} = \min_{j\in T} \gamma_{\rstar}(\widehat{\Sigma}_j)$ that $\hat{\gamma}_{\rstar} \leq \gamma_{\min}(\widehat D_{k})$. Moreover, by \Cref{lem:eigen_gap}, we have
\begin{align*}
\hat{\gamma}_{\rstar+1} = \hat{\gamma}_{\rstar+1}(\widehat \Sigma_{k}) = \norm{ (\widehat{U}^{\star}_{\perp})^\top 
\widehat \Sigma_{k} \widehat{U}^{\star}_{\perp} }.
\end{align*}
It follows from $\widehat \Sigma_{k} \succeq \widehat U_{k} \widehat D_{k} \widehat U_{k}^\top$ that
\begin{align}\label{eq::ratio}
\hat{\gamma}_{\rstar+1} 
\geq \gamma_{\min}(\widehat{D}_{k}) \norm{ (\widehat{U}^{\star}_{\perp})^\top \widehat U_{k} }^2 
\geq \hat{\gamma}_{\rstar} \, \norm{ (\widehat{U}^{\star}_{\perp})^\top \widehat U_{k} }^2\ \implies \ \norm{ (\widehat{U}^{\star}_{\perp})^\top \widehat U_{k} } \leq \sqrt{\frac{\hat{\gamma}_{\rstar+1}}{\hat{\gamma}_{\rstar}}}.
\end{align}
Moreover, we have
\begin{align*}
& \widehat{U}^{\star} \widehat{D}^{\star} (\widehat{U}^{\star})^\top = \widehat{\Sigma}^{\star} = V^\top \Sigmastar V = V^\top \Ustar \Dstar {\Ustar}^\top V,\\
\implies & (V \widehat{U}^{\star}) \widehat{D}^{\star} (V \widehat{U}^{\star})^\top = V V^\top \Ustar \Dstar {\Ustar}^\top VV^\top = (I_d-V_\perp V_\perp^\top) \Ustar \Dstar {\Ustar}^\top (I_d-V_\perp V_\perp^\top)\\
\implies & {U_\perp^{\star}}^\top (V \widehat{U}^{\star}) \widehat{D}^{\star} (V \widehat{U}^{\star})^\top U_\perp^{\star} = {U_\perp^{\star}}^\top V_\perp V_\perp^\top \Ustar \Dstar {\Ustar}^\top V_\perp V_\perp^\top U_\perp^{\star}.
\end{align*}
On the other hand, one can write
\begin{align*}
\begin{aligned}
\norm{ {U_\perp^{\star}}^\top (V \widehat{U}^{\star}) \widehat{D}^{\star} (V \widehat{U}^{\star})^\top U_\perp^{\star}} &\geq \gamma_{\min}(\widehat{D}^{\star})\norm{ (V \widehat{U}^{\star})^\top U_\perp^{\star} }^2,\\
\norm{ {U_\perp^{\star}}^\top V_\perp V_\perp^\top \Ustar \Dstar {\Ustar}^\top V_\perp V_\perp^\top U_\perp^{\star} } &\leq \norm{ V_\perp^\top \Ustar \Dstar {\Ustar}^\top V_\perp } = \norm{ V_\perp^\top \Ustar {\Dstar}^{1/2} }^2,
\end{aligned}
\end{align*}
which leads to 
\begin{align}\label{eq::UBB}
    \norm{ (V \widehat{U}^{\star})^\top U_\perp^{\star} } \leq \frac{\norm{ V_\perp^\top \Ustar {\Dstar}^{1/2} }}{\sqrt{\gamma_{\min}(\widehat{D}^{\star})}}.
\end{align}
Note that $V\widehat U_{k}, V \widehat{U}^{\star} \in O(d, \rstar)$. Moreover, $V  \widehat{U}^{\star}_{\perp} \in O(d, B-\rstar)$ is the orthogonal complement of $V \widehat{U}^{\star}$. Therefore, we have
\begin{align*}
&\norm{ (V\widehat U_{k})(V\widehat U_{k})^\top - \Ustar {\Ustar}^\top }\\
&\leq \norm{ (V\widehat U_{k})(V\widehat U_{k})^\top - (V \widehat{U}^{\star})(V \widehat{U}^{\star})^\top } + \norm{ (V \widehat{U}^{\star})(V \widehat{U}^{\star})^\top - \Ustar {\Ustar}^\top } \\
&\stackrel{(a)}{=} \norm{(V\widehat U_{k})^\top (V \widehat{U}^{\star}_{\perp}) } + \norm{(V \widehat{U}^{\star})^\top U_\perp^{\star} }\\
&\stackrel{(b)}{\leq} \sqrt{\frac{\hat{\gamma}_{\rstar+1}}{\hat{\gamma}_{\rstar}}} + \frac{\norm{V_\perp^\top \Ustar {\Dstar}^{1/2}}}{\sqrt{\gamma_{\min}(\widehat{D}^{\star})}}\\
&\stackrel{(c)}{\leq} \frac{3 \sqrt{2} \sqrt{\norm{\Sigma_\xi}}}{\frac{t_0}{2}\left(\sqrt{\gammastar_{\min}} - \norm{V_\perp^\top \Ustar {\Dstar}^{1/2}}\right) - 3 \sqrt{\norm{\Sigma_\xi}}}
+ \frac{\norm{V_\perp^\top \Ustar {\Dstar}^{1/2}}}{\sqrt{\gammastar_{\min}} - \norm{V_\perp^\top \Ustar {\Dstar}^{1/2}}}\\
&\stackrel{(d)}{\leq} \frac{3 \sqrt{2} \sqrt{\norm{\Sigma_\xi}}}{\frac{t_0^2}{2(t_0 + C)}\sqrt{\gammastar_{\min}} - \frac{t_0^2}{8(t_0 + C)}\sqrt{\gammastar_{\min}}}
+ \frac{\frac{4C}{t_0^2}\left(5\sqrt{\rstar\norm{\Sigma_\xi}} + \sqrt{\tr(\Sigma_\xi)}\right)}{\frac{t_0}{t_0 + C}\sqrt{\gammastar_{\min}}}\\
&= \left( \frac{8\sqrt{2}(t_0 + C)}{t_0^2} + \frac{20(t_0+C)C\sqrt{\rstar}}{t_0^3} \right) \sqrt{\frac{\norm{\Sigma_\xi}}{\gammastar_{\min}}} + \frac{4(t_0 + C)C}{t_0^3}\sqrt{\frac{\tr(\Sigma_\xi)}{\gammastar_{\min}}}\\
&\stackrel{(e)}{\leq} \frac{4(t_0 + C)C}{t_0^3} \left( 6\sqrt{\frac{ \rstar \norm{\Sigma_\xi}}{\gammastar_{\min}}}  + \sqrt{\frac{\tr(\Sigma_\xi)}{\gammastar_{\min}}}\right)
\end{align*}
with probability at least $1-3\delta$. Here, $(a)$ follows from \Cref{lem:dk}, $(b)$ follows from \Cref{eq::ratio} and \Cref{eq::UBB}, $(c)$ follows from \Cref{prop:eigen_gap}, \Cref{lem:min_max}, $(d)$ follows from \Cref{thm::first-stage-formal} and the condition $6\sqrt{\norm{\Sigma_\xi}} \leq \sqrt{\tr(\Sigma_\xi)} + 5\sqrt{\rstar \norm{\Sigma_\xi}} \leq \frac{t_0^2}{4(t_0+C)} \sqrt{\gammastar_{\min}}$, $(e)$ follows from the fact that $\rstar \geq 2$, $C \geq 2.2$ and $t_0 \leq 1$.
\end{proof}

\section{Conclusions}
We address robust subspace recovery (RSR) under both Gaussian noise and strong adversarial corruptions. Given a limited set of noisy samples—some altered by an adaptive adversary—our goal is to recover a low-dimensional subspace that approximately captures most of the clean data. We propose RANSAC+, a two-stage refinement of the classic RANSAC algorithm that addresses its key limitations. Unlike RANSAC, RANSAC+ is provably robust to both noise and adversaries, achieves near-optimal sample complexity without prior knowledge of the subspace dimension, and is significantly more efficient.

\section*{Acknowledgements}
This research is supported, in part, by NSF CAREER Award CCF-2337776, NSF Award DMS-2152776, and ONR Award N00014-22-1-2127.


\bibliography{references}

\end{document}